\documentclass{article}
\PassOptionsToPackage{numbers, compress}{natbib}
\usepackage[preprint]{neurips_2021}
\usepackage[utf8]{inputenc} % allow utf-8 input
\usepackage[T1]{fontenc}    % use 8-bit T1 fonts
\usepackage{hyperref}       % hyperlinks
\usepackage{url}            % simple URL typesetting
\usepackage{booktabs}       % professional-quality tables
\usepackage{amsfonts}       % blackboard math symbols
\usepackage{nicefrac}       % compact symbols for 1/2, etc.
\usepackage{microtype}      % microtypography
\usepackage{xcolor}         % colors
\usepackage{natbib}
\usepackage{amssymb}% http://ctan.org/pkg/amssymb
\usepackage{pifont}% http://ctan.org/pkg/pifont
\usepackage{amsmath,amsfonts,bm}

\def\eqref#1{equation~\ref{#1}}
\def\1{\bm{1}}

\DeclareMathAlphabet{\mathsfit}{\encodingdefault}{\sfdefault}{m}{sl}
\SetMathAlphabet{\mathsfit}{bold}{\encodingdefault}{\sfdefault}{bx}{n}

\newcommand{\E}{\mathbb{E}}

\DeclareMathOperator*{\argmax}{arg\,max}
\DeclareMathOperator*{\argmin}{arg\,min}

\usepackage{microtype}
\usepackage{graphicx}
\usepackage{subfigure}
\usepackage{booktabs} % for professional tables

\usepackage{hyperref}       % hyperlinks
\usepackage{url}            % simple URL typesetting
\usepackage{booktabs}       % professional-quality tables
\usepackage{amsfonts}       % blackboard math symbols
\usepackage{nicefrac}       % compact symbols for 1/2, etc.
\usepackage{microtype}      % microtypography
\usepackage{graphicx}
\usepackage{subfigure}
\usepackage{wrapfig}
\usepackage{setspace}
\usepackage{graphicx}
\usepackage{subfigure}
\usepackage{amsmath,amsfonts,amsthm,amssymb}
\usepackage{mathrsfs}
\usepackage{algorithm}% http://ctan.org/pkg/algorithms
\usepackage{algcompatible}% http://ctan.org/pkg/algorithmicx
\algnewcommand\algorithmicreturn{\textbf{return}}
\algnewcommand\RETURN{\State \algorithmicreturn}%

\theoremstyle{definition}

\newtheorem{theorem}{Theorem}
\newtheorem{lemma}{Lemma}
\newtheorem{remark}{Remark}
\newtheorem{prop}{Proposition}

\newcommand{\vect}[1]{\boldsymbol{#1}}
\newcommand\numberthis{\addtocounter{equation}{1}\tag{\theequation}}
\usepackage{graphicx}
\usepackage{subfigure}
\usepackage{amsmath}
\usepackage{multirow}
\usepackage{wrapfig}

\usepackage{xcolor}
\newcount\comments  % 0 suppresses notes to selves in text
\newcommand{\genComment}[2]{\ifnum\comments=1{\textcolor{#1}{\textsf{\footnotesize #2}}}\fi}

\title{Zeroth-Order Supervised Policy Improvement}

\author{%
Hao Sun\textsuperscript{1}\thanks{sh018@ie.cuhk.edu.hk},
Ziping Xu\textsuperscript{2},
Yuhang Song\textsuperscript{3},
Meng Fang\textsuperscript{4},
Jiechao Xiong\textsuperscript{4},
Bo Dai\textsuperscript{5},
Bolei Zhou\textsuperscript{1} \\
\textsuperscript{1}The Chinese University of Hong Kong,
\textsuperscript{2}University of Michigan,\\
\textsuperscript{3}University of Oxford,
\textsuperscript{4}Tencent,
\textsuperscript{5}Nanyang Technological University
}

\begin{document}

\maketitle

\begin{abstract}
Policy gradient (PG) algorithms have been widely used in reinforcement learning (RL). However, PG algorithms rely on exploiting the value function being learned with the first-order update locally, which results in limited sample efficiency. In this work, we propose an alternative method called Zeroth-Order Supervised Policy Improvement (ZOSPI). ZOSPI exploits the estimated value function $Q$ globally while preserving the local exploitation of the PG methods based on zeroth-order policy optimization. This learning paradigm follows Q-learning but overcomes the difficulty of efficiently operating argmax in continuous action space. 
It finds max-valued action within a small number of samples. The policy learning of ZOSPI has two steps: First, it samples actions and evaluates those actions with a learned value estimator, and then it learns to perform the action with the highest value through supervised learning. We further demonstrate such a supervised learning framework can learn multi-modal policies. Experiments show that ZOSPI achieves competitive results on the continuous control benchmarks with a remarkable sample efficiency.\footnote{Code is included in the supplemental materials.} 
%policy can be implemented with gradient-free non-parametric models besides normal neural network approaches. 
% We prove that with a good function structure, the zeroth-order optimization strategy combining both local and global samplings can find the global minima within a polynomial number of samples. 
% To improve the exploration efficiency in unknown environments, ZOSPI is further combined with bootstrapped $Q$ networks.
\end{abstract}
\section{Introduction}
Model-free Reinforcement Learning achieves great successes in many challenging tasks~\cite{mnih2015human,vinyals2019grandmaster,pachockiopenai}, however one hurdle for its applicability to real-world control problems is the low sample efficiency. 
To improve the sample efficiency, off-policy methods~\cite{degris2012off,gu2016q,wang2016sample,lillicrap2015continuous,fujimoto2018addressing} reuse the experiences generated by previous policies to optimize the current policy, therefore can obtain a much higher sample efficiency than the on-policy methods~\cite{schulman2015trust,schulman2017proximal}. Alternatively, SAC~\cite{haarnoja2018soft} improves sample efficiency by conducting more active exploration with maximum entropy regularizer~\cite{haarnoja2017reinforcement} to the off-policy actor critic~\cite{degris2012off,zhang2019generalized}.%, and also results in the state-of-the-art asymptotic performance.
% proposes to regularize off-policy actor-critic by the maximum entropy RL framework for better exploration, which results in a much improved sample efficiency and the state-of-the-art asymptotic performance.
OAC~\cite{ciosek2019better} further improves SAC by combining it with the Upper Confidence Bound heuristics~\cite{brafman2002r} to incentivize more informed exploration.
However, these previous methods all rely on Gaussian-parameterized policies and local exploration strategies that simply add noises at the action space, thus they might still converge to poor sub-optimal solutions, which has been shown in the work of~\cite{tessler2019distributional}.
%by relying on a Gaussian policy and a local exploration strategy from simply adding noises to the action space, these method might still lead to sub-optimal solutions as pointed by~\cite{tessler2019distributional}.

%\textcolor{red}{motivation: global exploration and exploitation}

%Previous methods focus on improving the sample efficiency through the perturbation in the action space,

%Motivated by 
In this work we aim to explore an alternative approach to policy gradient paradigm based on supervised learning. Our method carries out non-local exploration through global value-function exploitation to achieve higher sample efficiency for continuous control tasks. % as Q-learning but still be able to tackle the continuous control problems.
%\footnote{**m:One problem with applying Q-learning to continuous control is that a single sub-optimal action will not prevent a high value action so that affects sample efficiency.} 
Specifically, to better exploit the learned value function $Q$, we propose to search the action space globally for a better target action. This is in contrast to previous policy gradient methods that only utilize the local information of the learned value functions (e.g. the Jacobian matrix~\cite{silver2014deterministic}).
%we propose to better exploit the learned value functions $Q$, where we search globally for a better action rather than only utilize its local information or the Jacobian matrix used in previous policy gradient methods~\cite{silver2014deterministic}. %\ziping{Discuss} %\bz{Here I will suggest to describe high-levelly how you exploit Q.}
%It's worth noting that the concept of exploitation is different from the meaning in the exploration-exploitation dilemma \cite{sutton1998introduction}. 
%Here we refer to exploitation as mining more information from existing samples. \ziping{This is an unclear expression.}
The idea behind our work is related to the value-based policy gradient methods~\cite{lillicrap2015continuous,fujimoto2018addressing}, where the policy gradient optimization step takes the role of finding a well-performing action given a learned state-action value function. In previous work, the policy gradient step tackles the curse of dimensionality for deep Q-learning since it is intractable to directly search for the maximal value in the continuous action space~\cite{mnih2015human}. Differently, we circumvent searching for the action with the maximal value in the continuous action space by finding the max-valued action within a small set of sampled actions. 
%of action space with limited number of instances.

%Inspired by the works of evolution strategies~\cite{salimans2017evolution,conti2018improving,mania2018simple} that adopt zeroth-order methods in the parameter space, we apply the zeroth-order method to the action space and then update policy through supervised learning. 
To perform a global exploitation of the learned value function $Q$, we propose to apply a zeroth-order optimization scheme and update the target policy through supervised learning, which is inspired by works of evolution strategies~\cite{salimans2017evolution,conti2018improving,mania2018simple} that adopt zeroth-order optimizations in the parameter space. 
Different from the standard policy gradient, combining the zeroth-order optimization with supervised learning forms a new way of policy update. Such an update avoids the local improvement of policy gradient: when policy gradient is applied, the target policy uses policy gradient to adjust its predictions according to the deterministic policy gradient theorem~\cite{silver2014deterministic}, but such adjustments can only lead to local improvements and may induce sub-optimal policies due to the non-convexity of the policy function~\cite{tessler2019distributional}; on the contrary, our integrated approach of zeroth-order optimization and supervised learning 
%our sample-based zeroth-order method with supervised learning 
can greatly improve the non-convex policy optimization and more likely escape the potential local minimum.
% Specifically, \textcolor{red}{how to exploit Q globally?}
% Such a globally exploitation of the value function $Q$ 
% can not only \textcolor{red}{better rewards?} but also improve the sample efficiency.
%Therefore, intuitively such globally value function exploitation on will also help in improving sample efficiency, which provides the key insight of our proposed method.

%We summarize our main contributions as follows: 
To summarize, in this work we introduce a simple yet effective policy optimization method called Zeroth-Order Supervised Policy Improvement (ZOSPI), as an alternative to the policy gradient approaches for continuous control. The policy from ZOSPI exploits global information of the learned value function $Q$ and updates itself through sample-based supervised learning. Besides, we demonstrate learning multi-modal policy based on ZOSPI to show the flexibility of the proposed supervised learning paradigm. The experiments show the improved empirical performance of ZOSPI over several popular policy gradient methods on the continuous control benchmarks, in terms of both higher sample efficiency and asymptotic performance.
%\begin{enumerate}
%    \item We introduce a simple and effective policy optimization method, Zeroth-Order Supervised Policy Improvement (ZOSPI), as an alternative to the policy gradient approaches for continuous control. The policy from ZOSPI exploits global information of the learned value function $Q$ and updates itself through sample-based supervised learning. 
%    \item We exhibit multi-modal extensions upon ZOSPI to show the flexibility of the proposed supervised learning paradigm.
%    \item We demonstrate the improved empirical performance of ZOSPI by comparing it against several state-of-the-art policy gradient methods on the continuous control benchmarks, where our proposed method achieves both higher sample efficiency and asymptotic performance.
%\end{enumerate}

%we expose two of potential future extensions based on the proposed self-supervised approach, namely the combination of optimistic explorations and the compatibility with Gaussian Processes policies. 

% in order to get better estimation of the value function $Q$, we combine ZOSPI with optimistic exploration strategies to reduce the estimation error. 
% Finally, we verify the exploration improvement of ZOSPI in a diagnostic environment named Four-Solution-Maze, where only ZOSPI with optimistic exploration is able to find the optimal solution.

\section{Related Work}
\subsection{Policy Gradient Methods}
The policy gradient methods solve an MDP by directly optimizing the policy to maximize the cumulative reward~\cite{williams1992simple,sutton1998reinforcement}. 
While the prominent on-policy policy gradient methods like TRPO~\cite{schulman2015trust} and PPO~\cite{schulman2017proximal} improve the learning stability of vanilla policy gradient~\cite{williams1992simple} via trust region updates, 
the off-policy methods such as ~\cite{degris2012off, lillicrap2015continuous,wang2016sample}
employ an experience replay mechanism to achieve higher sample efficiency. 
The work of TD3~\cite{fujimoto2018addressing} further addresses the function approximation error and boosts the stability of DDPG with several improvements.  
Another line of works is the combination of policy gradient methods and the max-entropy regularizer, which leads to better exploration and stable asymptotic performances~\cite{haarnoja2017reinforcement,haarnoja2018soft}. 
All of these approaches adopt function approximators~\cite{sutton2000policy} for state or state-action value estimation and take directionally-uninformed Gaussian as the policy parameterization, which lead to a local exploration behavior~\cite{ciosek2019better,tessler2019distributional}.

%REINFORCE, Value Function Approximation, Actor-Critic, DPG, DDPG, TD3, SAC, TRPO, PPO

\subsection{RL by Supervised Learning}
Iterative supervised learning and self-imitate learning are becoming an alternative approach for model-free RL. Instead of applying policy gradient for policy improvement, methods based on supervised methods update policies by minimizing the mean square error between target actions and current actions predicted by a policy network~\cite{sun2019policy}, or alternatively by maximizing the likelihood for a stochastic policy class~\cite{ghosh2019learning}. While those previous works focus on the Goal-Conditioned tasks in RL, in this work we aim to tackle more general RL tasks. Some other works use supervised learning to optimize the policy towards manually selected policies to achieve better training stability under both offline~\cite{wang2018exponentially} and online settings~\cite{zhang2019policy,abdolmaleki2018relative,song2019v}. Differently, in our work the policy learning is based on a much simpler formulation than the previous attempts~\cite{abdolmaleki2018relative,lim2018actor,simmons2019q}, and it does not rely on any expert data to achieve competitive performance.

%\subsection{Policy Distillation}

\subsection{Zeroth-Order Methods}
Zeroth-order optimization methods, also called gradient-free methods, are widely used when it is difficult to compute the gradients. They approximate the local gradient with random samples around the current estimate. The works in~\cite{wang2017stochastic,golovin2019gradientless} show that a local zeroth-order optimization method has a convergence rate that depends logarithmically on the ambient dimension of the problem under some sparsity assumptions. It can also efficiently escape saddle points
in non-convex optimizations~\cite{vlatakis2019efficiently,bai2020escaping}. In RL, many studies have verified an improved sample efficiency of zeroth-order optimization \cite{usunier2016episodic,mania2018simple,salimans2017evolution}. In this work we provide a novel way of combining the local sampling and the global sampling to ensure that our algorithm can approximate the gradient descent locally and also find a better global region. 

\section{Preliminaries}
\subsection{Markov Decision Processes}
We consider the deterministic Markov Decision Process (MDP) with continuous state and action spaces in the discounted infinite-horizon setting. Such MDPs can be denoted as $\mathcal{M} = (\mathcal{S}, \mathcal{A}, \mathcal{T}, r, \gamma)$, where the state space $\mathcal{S}$ and the action space $\mathcal{A}$ are continuous, and the unknown state transition probability representing the transition dynamics is denoted by $\mathcal{T}: \mathcal{S} \times \mathcal{A} \mapsto \mathcal{S}$. $r: \mathcal{S} \times \mathcal{A} \mapsto [0, 1]$ is the reward function and $\gamma \in [0, 1]$ is the discount factor. An MDP $\mathcal{M}$ and a learning algorithm
operating on $\mathcal{M}$ with an arbitrary initial state $s_0 \in \mathcal{S}$ constitute a stochastic process described sequentially by the state $s_t$ visited at time step $t$, the action $a_t$ chosen by the algorithm at step $t$, the reward $r_t = r(s_t, a_t)$ and the next state $s_{t+1} = \mathcal{T}(s_t, a_t)$ for any $t = 0, \dots, T$. Let $H_t = \{s_0, a_0, r_0, \dots, s_{t}, a_{t}, r_{t}\}$ be the trajectory up to time $t$. Our algorithm finds 
the policy that maximizes the discounted cumulative rewards $\sum_{t=0}^T \gamma^t r_t$.
% \ziping{Objective}
Our work follows the general Actor-Critic framework~\cite{konda2000actor,peters2008natural,degris2012off,wang2016sample}, which learns in an unknown environment using a value network denoted by $Q_{w_t}: \mathcal{S} \times \mathcal{A} \mapsto \mathbb{R}$ for estimating $Q$ values and a policy network for learning the behavior policy $\pi_{\theta_t}: \mathcal{S} \mapsto \mathcal{A}$. Here $w_t$ and $\theta_t$ are the parameters of these two networks at step $t$, respectively. %We omit the subscript $t$ of $w_t$ and $\theta_t$ in the rest of this work for conciseness.

\subsection{Mixture Density Networks}
%\textcolor{red}{This section looks awkward to me...there needs some background introduction on why MDN is necessary to be introduced here.}
%\ziping{We need to explain why we are mentioning MDN here. I added a sentence. Please read.} 

%Due to the multi-modal property of the optimal policy in some environments, we introduce the Mixture Density Networks (MDN) to the learning paradigm proposed in this work.
The optimal policies in many environments tend to have multi-modal property. Mixture Density Networks (MDN), previously proposed in~\cite{bishop1994mixture} to solve the stochastic prediction problem, is a potential multi-modal policy function to be introduced in this work. One challenge here is that previous policy gradient methods are not compatible with multi-modal policies based on MDN.
 While in normal regression tasks, the learning objective is to maximize the log-likelihood of observed data given the current parameters of a predictive model, MDN uses the mixture of Gaussian with multiple parameters as $\mathcal{D} = \{X_i, Y_i\}_\mathbb{N}$. Then the likelihood is given by
\begin{equation}
   \mathcal{L} = \sum_{i=1}^{i=\mathbb{N}}\mathcal{L}(Y_i|X_i, \theta) = \sum_{i=1}^{i=\mathbb{N}} \sum^K_{k=1} w_k(X_i,\theta) \phi(Y_i| \mu_k(X_i,\theta), \sigma_k(X_i,\theta)),
\end{equation}
where $\theta$ denotes the parameters of neural networks with three branches of outputs $\{w_{i},\mu_{i},\sigma_{i}\}_K$, $\phi$ is the probability density function of normal distribution and $\sum_{i=1}^{K} w_i = 1$.

\section{Zeroth-Order Supervised Continuous Control}
Q-learning in the tabular setting finds the best discrete action given the current state, which can be difficult in the continuous action space due to the non-convexity of $Q$. A policy network is thus trained to approximate the optimal action. However, gradient-based algorithm may be trapped by local minima or saddle point depending on the initialization. In this section, we introduce a hybrid method with a global policy and a perturbation policy. The global policy predicts a coarse action by supervised learning on uniformly sampled actions and the perturbation network iterates based on the coarse prediction. In the following section, we give a motivating example to demonstrate the benefits of the hybrid framework.

%In most of previous policy gradient methods, the policy class is selected to be Gaussian in consideration of both exploration and computational tractability, while, in this work, we consider the deterministic policy class which is simpler and easier to learn as presented in~\cite{silver2014deterministic}. 

\subsection{A Motivating Example}
\begin{figure}[t]
\centering
\subfigure[Policy Gradient]{
\begin{minipage}[htbp]{0.33\linewidth}
	\centering
	\includegraphics[width=1\linewidth]{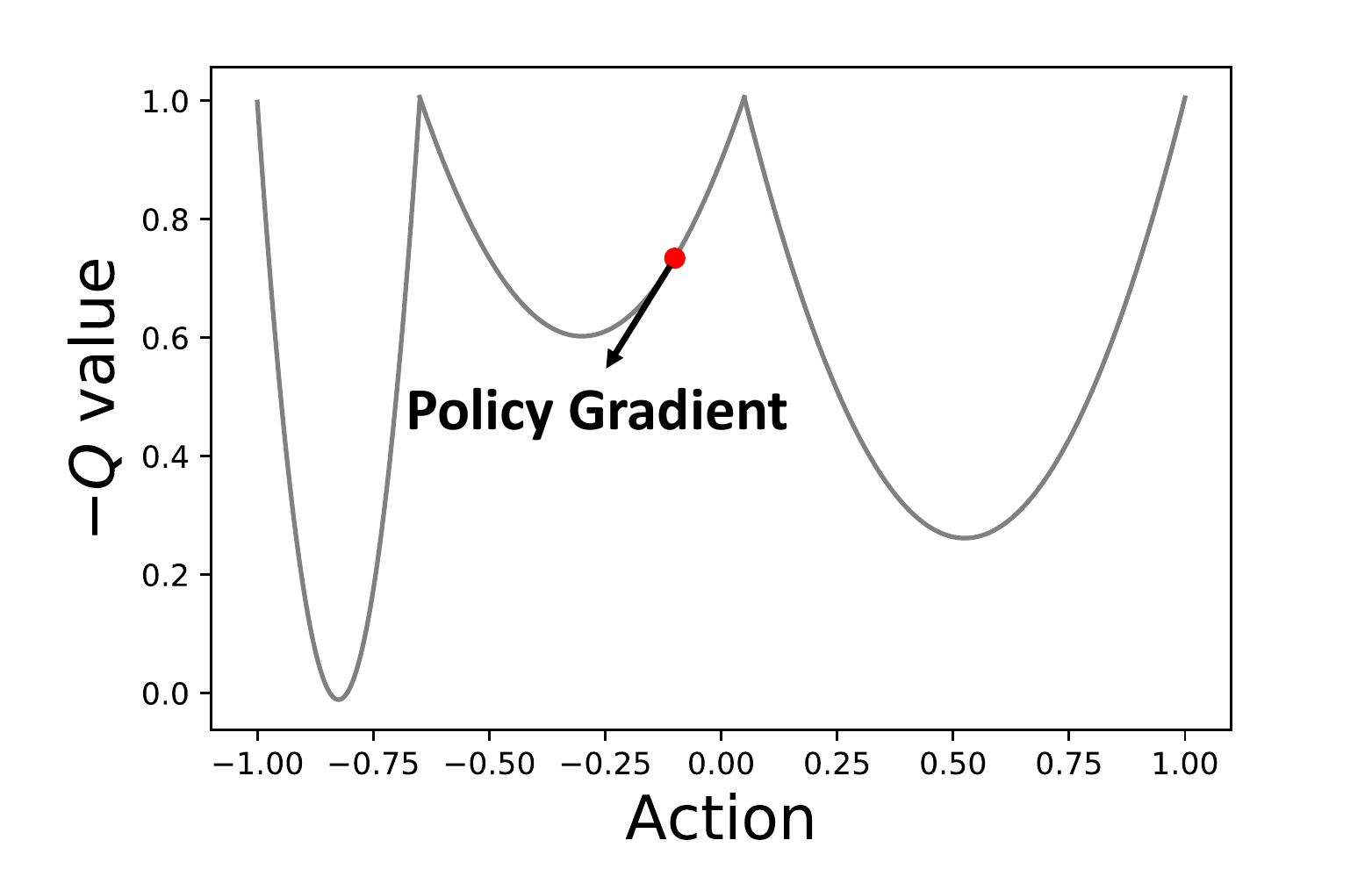}
\end{minipage}}%
\subfigure[ZOSPI]{
\begin{minipage}[htbp]{0.33\linewidth}
	\centering
	\includegraphics[width=1\linewidth]{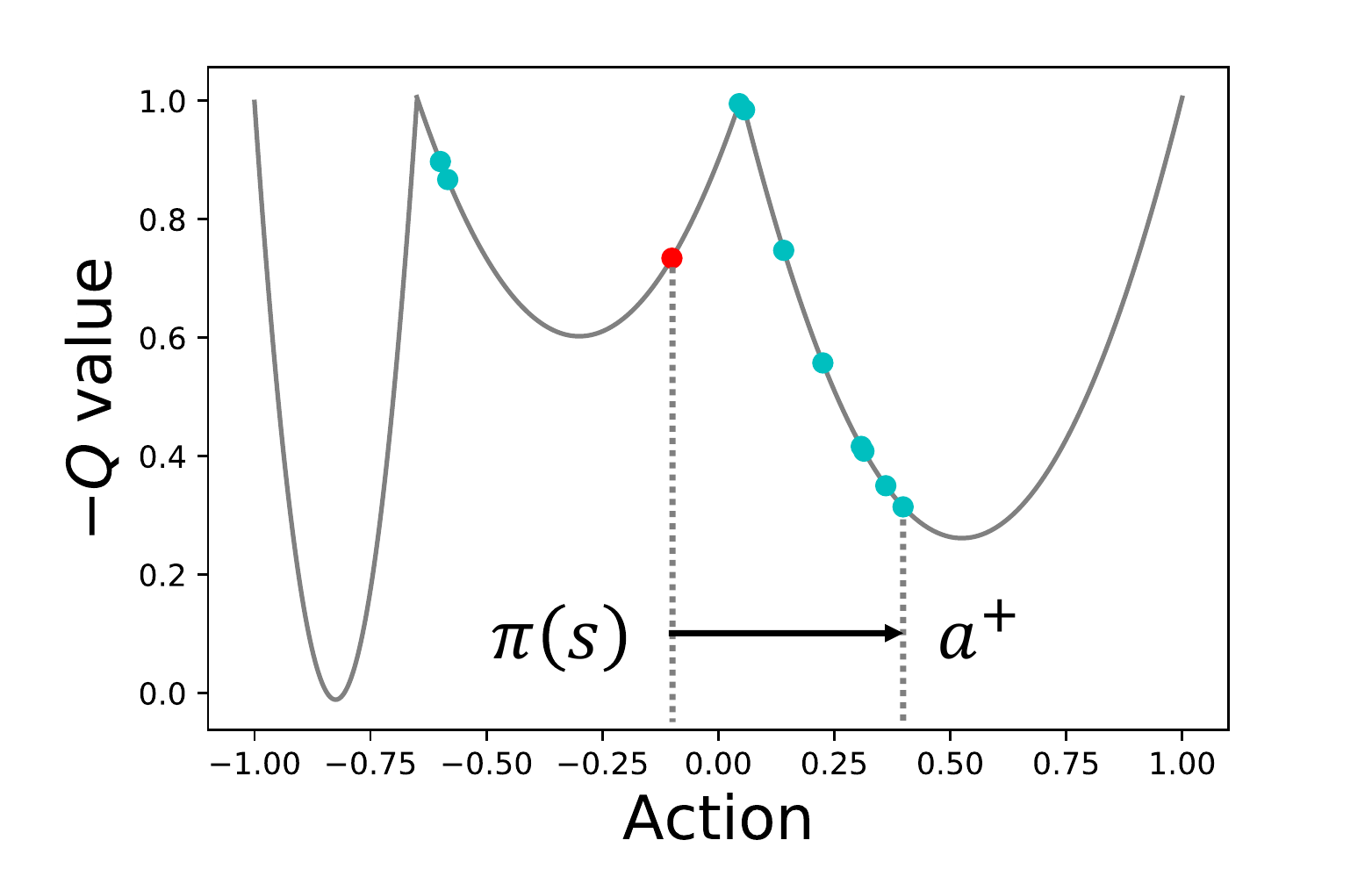}
\end{minipage}}%
\subfigure[Simulation]{
\begin{minipage}[htbp]{0.33\linewidth}
	\centering
	\includegraphics[width=1\linewidth]{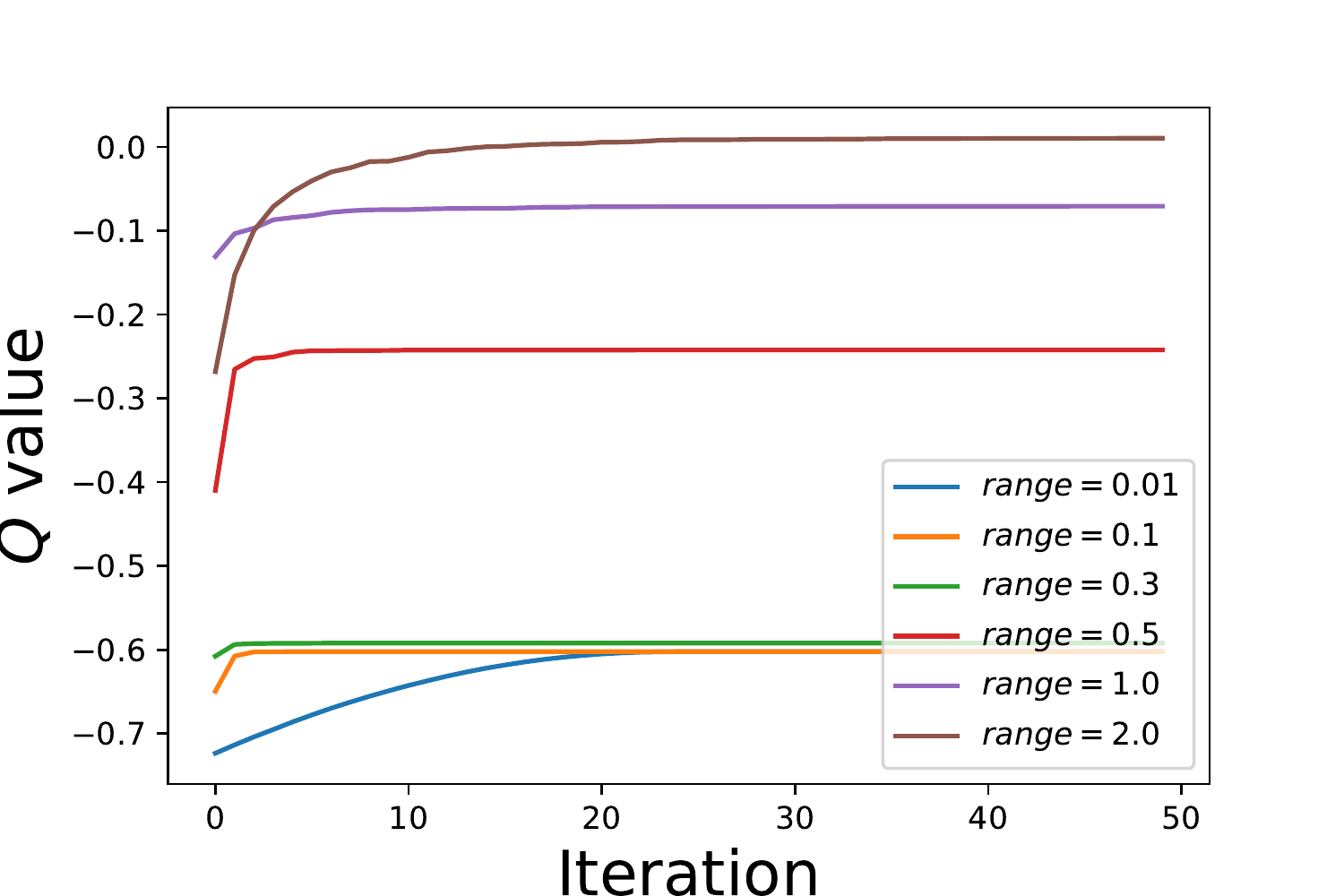}
\end{minipage}}
\caption{(a) Landscape of $Q$ value for a $1$-dim continuous control task. Policy gradient methods optimize the policy according to the local information of $Q$. (b) For the same task, ZOSPI directly updates the predicted actions to the sampled action with the largest $Q$ value. (c) Simulation results, where in each optimization iteration $10$ actions are uniformly sampled under different ranges. The reported results are averaged over $100$ random seeds. It can be seen that a larger random sample range improves the chance of finding global optima. Similar phenomenon also exist in practice as shown in Appendix \ref{vis_Q}.}
\label{fig_1}
\end{figure}
%We start with a motivating example to demonstrate the benefit of applying zeroth-order methods in policy optimization step given $Q$ function. 

% \bz{Give a bit more detailed description about the setup of this toy example, such as one dimentional continous action bablbal. Check the description for the following red dot and blue dots.}
Figure~\ref{fig_1} shows a motivating example to demonstrate the benefits of applying zeroth-order optimization to policy updates. Consider we have a learned $Q$ function with multiple local optima. Here we assume the conventional estimation of $Q$ function is sufficient for a global exploitation \cite{fujimoto2018addressing,haarnoja2018soft} and we will discuss an improved estimation method in the Appendix. Our deterministic policy selects a certain action at this state, denoted as the red dot in Figure~\ref{fig_1}(a). In deterministic policy gradient methods~\cite{silver2014deterministic,lillicrap2015continuous}, the policy gradient is conducted according to the chain rule to update the policy parameter $\theta$ with regard to $Q$-value by timing up the Jacobian matrix $\nabla_\theta \pi_\theta(s)$ and the derivative of $Q$, \emph{i.e.}, $\nabla_a Q(s,a)$. Consequently, the policy gradient can only guarantee to find a local minima, and similar local improvement behaviors are also observed in stochastic policy gradient methods like PPO and SAC~\cite{schulman2017proximal,haarnoja2018soft,tessler2019distributional,ciosek2019better}. 
Instead, if we can sample sufficient {\it random} actions in a broader range of the action space, denoted as blue dots in Figure~\ref{fig_1}(b), and then evaluate their values respectively through the learned $Q$ estimator, it is possible to find a better initialization, from which the policy gradient can more likely find the global minima. Figure~\ref{fig_1}(c) shows the simulation result using different sample ranges for the sample-based optimization starting from the red point. It is clear that a larger sample range improves the chance of finding the global optima. Utilizing such a global exploitation on the learned value function is the key insight of this work. 

% An example of how zeroth-order optimization leverages more global information compared with the first-order optimization methods. While policy gradient and zeroth-order method with locally samples can only achieve the local optimal, increasing the sample range of zeroth-order methods will improve the optimization results and able to achieve the global optimal point of the $Q$ function.

% \bz{Give some detailed description about Figure~\ref{fig_1}c in this new paragraph, if you want to keep the subfigure here. Also mention the range is an important factor in determining the level of exploration in the zeroth-order method.}

\subsection{Zeroth-Order Supervised Policy Improvement}

Based on the above motivation, we propose a framework called ZOSPI (Zeroth-Order Supervised Policy Improvement). ZOSPI consists of two policy networks: the global policy network that gives us a coarse estimate on the optimal action and a perturbation network that iterates based on the output of the global policy network using policy gradient.

\paragraph{Supervised learning for global policy network.} We denote the global policy network by $\pi_{\theta}$. At any step $t$, we sample a set of actions uniformly over the entire action space as well as a set of actions sampled from Gaussian distribution centered at current prediction $\pi_{\theta_t}$. We denote by $a^+_t$, the action that gives the highest $Q$ value with respect to current state $s_t$. Then we apply the supervised policy improvement that minimizes the $L_2$ distance between $a^+_t$ and $\pi_{\theta_t}(s_t)$, which gives the descent direction:
\begin{align*}
    \nabla_{\theta} \frac{1}{2}(a^+_t - \pi_{\theta_t}(s_t))^2 = (a^+_t - \pi_{\theta_t}(s_t)) \nabla_{\theta}\pi_{\theta_t}(s_t). \numberthis \label{equ:0th}
\end{align*}
The global samples can help finding the regions that are better in the whole space. The local samples from Gaussian distribution accelerate later-stage training when the prediction is accurate enough and most global samples are not as good as the current prediction.
The implementation detail is shown in Algorithm \ref{algo_zoa}.

\paragraph{Policy gradient for perturbation network.}
\begin{algorithm}[t]
	\caption{Policy Update with Zeroth-Order and First-Order Optimization}\label{algo_zoa}
	\begin{algorithmic}[1]
		\STATE \textbf{Require}
        \STATE Objective function $Q_s$, domain $\mathcal{A}$, current policy network $\pi_{\theta}$, perturbation network $\pi_{\phi}$ current point $a_0 = \pi_{\theta}(s)$, number of global samples $n_1$, number of local samples $n_2$ %, local scale $\eta > 0$ and step size $h$.
		\STATE \textbf{Global sampling}
		\STATE Sample $n_1$ points uniformly in the entire space by
		$$
		    a_{i} \sim \mathcal{U}_{\mathcal{A}}, \text{ for }  i = 1, \dots, n_1, \text{ where }\mathcal{U}_{\mathcal{A}} \text{ is the uniform distribution over } \mathcal{A}.
		$$
		\STATE \textbf{Local sampling}
		\STATE Sample $n_2$ points locally from a Gaussian centered at $a_0$ with a covariance matrix $\sigma^2 I$:
		$$
		    a_{i+n_1} \sim \mathcal{N}(a_0, \sigma^2 I), i = 1, \dots, n_2.
		$$
		\STATE Set $a^{+} = \argmax_{a \in \{a_0, \dots, a_{n_1+n_2}\}} Q_s(a) $.
		%\STATE \textbf{Return} $a^{+}$.
		\STATE Update policy $\pi_{\theta}$ according to Eq.(\ref{equ:0th})
		\STATE \textbf{Local perturbation}
		\STATE Update perturbation network according to Eq. (\ref{equ:pertb_update}).
	\end{algorithmic}
\end{algorithm}
%While ZOSPI is able to exploit the learned value function globally as demonstrated in Figure~\ref{fig_1}, the sample-based methods may need exponentially more number of samples when the dimension of action space increases when it comes to local exploitation of the learned value function. To overcome this drawback of zeroth-order method, 
We introduce another perturbation network, denoted by $\pi_{\phi_t}$, as~\cite{fujimoto2018off}. Such a perturbation network is parameterized by $\phi_t$ that performs fine-grained control on top of the global policy network $\pi_{\theta_t}$. Different from the $\pi_{\theta_t}$, the perturbation network $\pi_{\phi_t}$ takes both the current state $s_t$ and the predicted action $ \pi_{\theta_t}(s_t)$ as inputs, thus the final executed action is as follows:
\begin{equation}
    a_t = \pi_{\theta_t}(s_t) + \pi_{\phi_t}(s_t,\pi_{\theta_t}(s_t)).
\end{equation}

The range of the outputs for the perturbation network is limited to $0.05$ times the value of maximal action. Therefore, it is only able to \textit{perturb} the action provided by the policy network $\pi_{\theta_t}$. 

The perturbation network $\pi_{\phi_t}$ is trained with the policy gradient:
\begin{equation}
    \nabla_\phi J = \mathbb{E}[\nabla_{a'} Q_w(s_t, a')|_{a'=\pi_{\theta_t}(s_t) + \pi_{\phi_t}(s_t)} \nabla_\phi\pi_{\phi_t}(s_t)]. 
    \label{equ:pertb_update}
\end{equation}
The intuition behind such an empirical design can be drawn from the example of Figure~\ref{fig_1}: although the global sampling step as well as the zeroth-order method helps policy optimization escape sub-optimal regions of the non-convex value function, the first order method can help to optimize the decision afterwards inside the locally convex region more efficiently. Eq. (\ref{equ:pertb_update}) shows that the updates are accessed only through the gradient, $\nabla_a Q_w$.

\subsection{Analyses on the Benefits of Global Sampling}
\label{sec:thm}
In this section, we give some analyses on the benefits of global sampling in terms of the sampling efficiency. Our analyses does not consider the improvement from the local sampling set, with which the quality of supervised learning can only be even better.
\paragraph{Error rate of the $a^+_t$.} The performance of the supervised policy improvement heavily depends on the goodness of $a^+_t$, the best action in the sampling set. By assuming the continuity of the estimated $Q$ function, we give an upper bound on the error rate $\|a^+_t - a^*(s)\|$, where $a^*(s)$ is the true optimal action given the current state $s$.
\begin{lemma}
\label{lem:GE}
Assume the estimated $Q$ function $Q_w(s_t, \cdot)$ given any $s_t$ is $L$-Lipschitz with respect to action inputs. Let the action space $\mathcal{A} \subset \mathbb{R}^{d}$ for some positive integer $d$ and assume that any action $a \in \mathcal{A}$, $\|a\|_2 \leq k$ for $k > 0$. Then with a probability at least $1-\delta$, we have
$$
    \|a^+_t - a^*(s)\|_2 \leq 2k\sqrt{d}L \left(\frac{\log^2(n/\delta)}{n}\right)^{1/d}.
$$
\label{lem:a_t}
\end{lemma}
Though the dependence of $n$ is $1/n^{1/d}$, the estimation of $a_t^+$ is sufficient for a coarse prediction, because our prediction prediction $\pi_{\theta}$ aggregates all the historical information, which further reduces the error and we only need $\pi_{\theta}$ in a local convex region of the global optima. This rate can be much lower when we have a fairly better prediction and the best actions are given by the local sampling set instead.

\paragraph{Quality of $\pi_{\theta}$.} It is hard to directly evaluate the error rate of $\pi_{\theta}(s_t)$ because of the online stochastic gradient descent applied on a non-stationary distribution of states and a Q-function that is continuously being updated. Thus we take a step back and consider a simpler scenario. We assume a fixed $Q$ function and a fixed state distribution. We assume that the policy network is a linear model with $\theta \in \mathbb{R}^{p}$. We consider a global empirical risk minimizer, so called ERM, since it is easy to achieve for linear models. Let our dataset with $T$ samples be $\{(s_t, a_t^+)\}_{t = 1}^{T}$. Our ERM estimate is defined by 
$$
    \hat \theta = \argmin_{\theta \in \Theta} \sum_{t = 1}^T (a_t^+ - \pi_{\theta}(s_t))^2.
$$
\begin{theorem}
\label{thm:1}
If assumptions in Lemma \ref{lem:a_t} hold, with a high probability we have
$$
    \mathbb{E}\|\pi_{\hat\theta}(s) - a^*(s)\|_2^2 = \tilde{\mathcal{O}}\left(\frac{k \sqrt{d}pL}{n^{1/d}T} + \min_{\theta}\mathbb{E}\|\pi_{\theta}(s) - a^*(s)\|_2\right),
$$
where $\tilde{\mathcal{O}}$ hides all the constant and logarithmic terms.
\end{theorem}

Theorem \ref{thm:1} implies that though the error of $a_t^*$ at one step can be high, the error of our policy network can be further reduced by aggregating information from multiple steps. However, when $\pi_{\theta}$ is nonlinear functions, the quality may be worse than the case in our analyses. That is why we  need some local perturbation through policy gradient to further improve our policy. Missing proofs in this section are given by Appendix \ref{app:proof}.

\subsection{Multi-Modal Continuous Control with ZOSPI}
\label{sec_extension}
Different from standard policy gradient methods, the policy optimization step in ZOSPI can be considered as sampling-based supervised learning. Such a design enables many extensions of the proposed learning paradigm. In this section, we introduce the combination of ZOSPI and first-order method for stabilized training and the combination of ZOSPI and MDN for multi-modal policy learning.

% \begin{figure}[t]
% \vskip 0.2in
% \centering
% \begin{minipage}[htbp]{0.6\linewidth}
% 	\centering
% 	\includegraphics[width=1\linewidth]{figs/MDN.pdf}
% \end{minipage}%
% \caption{Illustration of how ZOSPI with MDN works when multiple optimal actions exist for a certain state: (left) Supermario is going to collect the mushroom, both $a_2$ and $a_3$ are the optimal actions for the current location. While deterministic policy gradient methods (green color) are only able to learn one of those optimal actions, ZOSPI with MDN is able to learn both. Previous policy gradient methods normally learn a Gaussian policy class.
% }
% \label{fig_mdn}
% \vskip -0.2in
% \end{figure}%
\subsubsection{Learning Multi-Modal Policies with Mixture Density Networks}
\label{sec_mdn}
\begin{wrapfigure}{l}{7cm}%靠文字内容的左侧
\vskip -0.1in
\centering
\includegraphics[width=0.5\columnwidth]{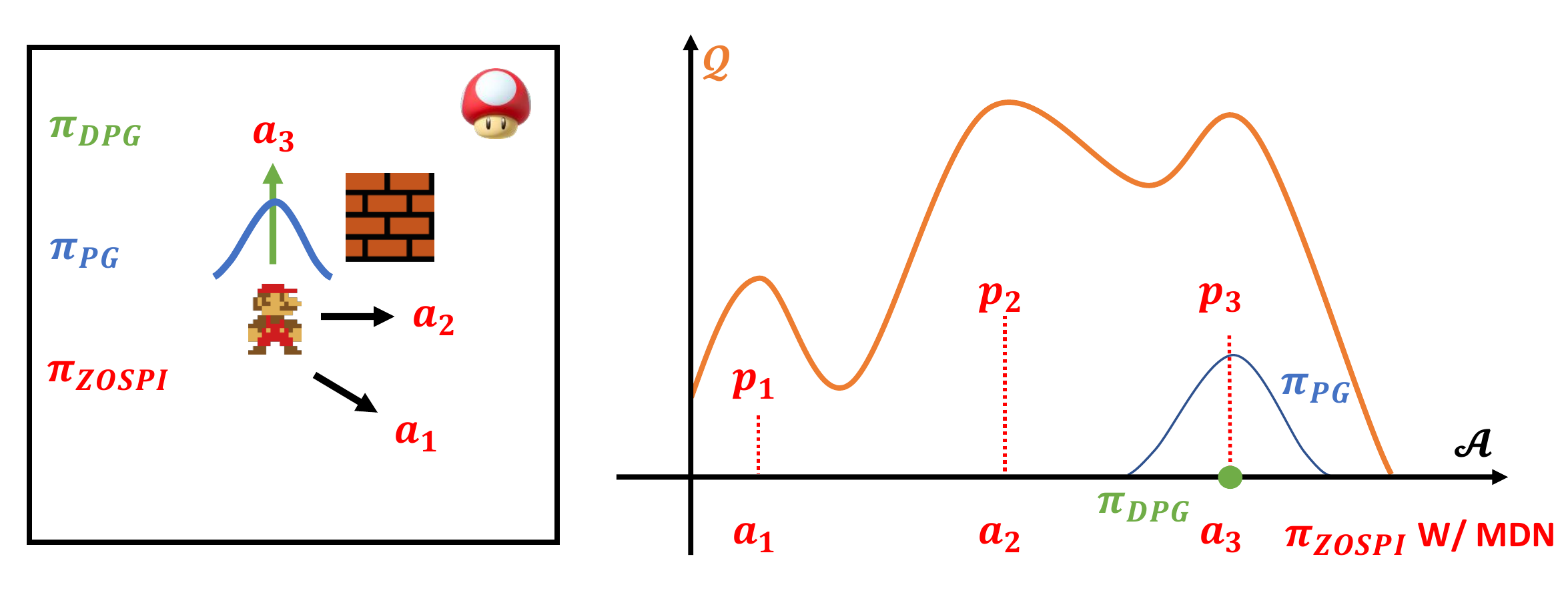}
\vskip -0.1in
\caption{Illustration of how ZOSPI with MDN works when multiple optimal actions exist for a certain state: (left) Supermario is going to collect the mushroom, both $a_2$ and $a_3$ are the optimal actions for the current location. While deterministic policy gradient methods (green color) are only able to learn one of those optimal actions, ZOSPI with MDN is able to learn both. Policy gradient methods normally learn a Gaussian policy class.}
%challenges previous RL algorithms.}
\label{fig_mdn}
\end{wrapfigure}

In the context of RL, there might be multiple optimal actions for some certain states, e.g., stepping up-ward and then right-ward may lead to the same state as stepping right and then upward, therefore both choices result in identical return in Figure~\ref{fig_mdn}. However, normal deterministic policy gradient methods can not capture such multi-modality due to the limitation that the DPG theorem is not applicable to a stochastic policy class.

In this section, we further integrate ZOSPI with the Mixture Density Networks (MDNs)~\cite{bishop1994mixture}, which was introduced for multi-modal regression.
Applying MDNs to ZOSPI leads to a more flexible and interpretable multi-modal policy class~\cite{tessler2019distributional}: different from normal Dirac policy parameterization used in TD3, ZOSPI with MDN predicts a mixture of Dirac policies, $i.e.$, the policy $\pi_{\mathrm{MDN}}(s)$ predicts $K$ choices for action $\vec{a}=\{a_1,...,a_K\}\in\mathcal{A}^{K}$, with their corresponding probability $\vec{p}=\{p_1,...,p_K\}\in \mathbb{R}^K$, and $\sum_{i}^K p_i = 1$. Then the action is sampled by
\begin{equation}
\pi_{\mathrm{MDN}}(s) = a_i, ~~\text{w.p.} ~~p_i,~ \text{for} ~~i = 1,...,K
\end{equation}
Thereby, instead of learning the mean value of multiple optimal actions, ZOSPI with MDN is able to learn multiple optimal actions. Different from previous stochastic multi-modal policies discussed in~\cite{haarnoja2018soft}, the learning of multi-modal policy does not aim to fit the entire $Q$ function distribution. Rather, ZOSPI with MDN focuses on learning the multi-modality in the best choices of actions, thus addresses the difficulty in generating multi-modal policies~\cite{haarnoja2017reinforcement,tessler2019distributional}. Figure~\ref{fig_mdn} illustrates the difference between ZOSPI with MDN and the previous policy gradient methods.% by learning multi-modal policies.

\begin{algorithm}[t]
	\caption{Zeroth-Order Supervised Policy Improvement (ZOSPI) }\label{algo_SPI}
	\begin{algorithmic}[1]
		\STATE \textbf{Require} \\
		~~ Number of epochs $M$, size of mini-batch $N$, momentum $\tau > 0$. \\
		~~ Random initialized policy network $\pi_{\theta}$, target policy network $\pi_{\theta'}$, $\theta'\leftarrow  \theta$. \\
		~~ Two random initialized $Q$ networks, and corresponding target networks, parameterized by $w_{1},w_{2},w'_{1},w'_{2}$. $w'_{i}\leftarrow w_{i}$.\\
		~~ Empty experience replay buffer $\mathcal{D} = \{\}$.
		\FOR{iteration $= 1,2,...$}
		\FOR{t $= 1,2,..., T$}
		\STATE $\#$ Interaction
		\STATE Run policy $\pi_{\theta}$ in environment, store transition tuples $(s_t,a_t,s_{t+1},r_t)$ into $\mathcal{D}$.
		\FOR{epoch $= 1,2,..., M$}
		\STATE Sample a mini-batch of transition tuples $\mathcal{D}^{\prime} = \{(s_{t_j},a_{t_j},s_{t_j+1},r_{t_j})\}_{j = 1}^{N}$.
		\STATE $\#$ Update $Q$
		    \STATE Calculate target $Q$ value $y_j = r_{t_j} + \min_{i=1,2} Q_{w^{\prime}_{i}}(s_{t_j+1},\pi_{\theta^{\prime}}(s_{t_j}))$.
		    \STATE Update $w_{i}$ with one step gradient descent on the loss $\sum_{j}(y_{j} - Q_{w^{\prime}_{i,}}(s_{t_j},a_{t_j}))^2$, $i=1,2$.
		\STATE $\#$ Update $\pi$
		\STATE Call Algorithm \ref{algo_zoa} for policy optimization to update $\theta$.
		  %  \ziping{How about we call Algorithm \ref{algo_zoa} here?}  
		  %  \STATE Calculate the predicted action $a_0 = \pi_{\theta'}(s)$
		  %  \STATE Sample actions $a_i\sim \mathcal{U}_\mathcal{A}$ (global sample) and from a \textcolor{red}{unit SH: is the term 'unit' correct?} ball $\mathcal{B}$ (local sample)
		  %  \STATE Select $a^+$ as the action with maximal $Q$ value: $a^+ = \arg\max_{\tilde{a}\in\{a_i\}\cup\{a_0\}}Q_1(s,\tilde{a}) $ 
		  %  \STATE Update policy network with Eq. (\ref{equ:0th})
		\ENDFOR
		\STATE $\theta'\leftarrow \tau \theta + (1-\tau) \theta'$; $w_{i}'\leftarrow \tau w_{i} + (1-\tau) w_{i}'$
% 		\STATE $w_{i} \leftarrow w_{i}$
% 		\STATE $\theta \leftarrow \theta$
		\ENDFOR
		\ENDFOR
	\end{algorithmic}
\end{algorithm}

\section{Experiments}
In this section, we conduct experiments on five MuJoCo locomotion benchmarks to demonstrate the effectiveness of the proposed method. Specifically, we validate the following statements: 
\begin{enumerate}
    \item If we use ZOSPI with locally sampled actions, the performance of ZOSPI should be the same as its policy gradient counterpart like TD3; if we increase the sampling range, ZOSPI can better exploit the $Q$ function thus find better solution than the policy gradient methods.
    \item If we continuously increase the sampling range, it will result in an uniform sampling, and the $Q$ function can be maximally exploited. 
    \item The perturbation network can help to improve the sample efficiency of the primal ZOSPI frame work purely based on zeroth-order optimization.
    \item The supervised learning policy update paradigm of ZOSPI permits it to be flexibly work with multi-modal policy class.
\end{enumerate}

% \begin{figure}[t]
% \begin{minipage}[htbp]{0.25\linewidth}
% 	\centering
% 	\includegraphics[width=1\linewidth]{figs/DDPG_MC.png}
% \end{minipage}%
% \begin{minipage}[htbp]{0.25\linewidth}
% 	\centering
% 	\includegraphics[width=1\linewidth]{figs/DDPG_TD.png}
% \end{minipage}
% \begin{minipage}[htbp]{0.25\linewidth}
% 	\centering
% 	\includegraphics[width=1\linewidth]{figs/sample_TD.png}
% \end{minipage}%
% \begin{minipage}[htbp]{0.25\linewidth}
% 	\centering
% 	\includegraphics[width=1\linewidth]{figs/sample_MC.png}
% \end{minipage}
% \caption{Comparison between sample and policy gradient methods, with MC or TD return}
% \label{ddpg_sample}
% \end{figure}

% \begin{figure}[t]
% 	\centering
% 	\includegraphics[width=0.7\linewidth]{figs/tempresult.png}
% \caption{Comparison between sample and policy gradient methods, with MC or TD return (10 seeds running)}
% \label{ddpg_sample_repeat}
% \end{figure}

\subsection{ZOSPI on the MuJoCo Locomotion Tasks.}
\begin{table}[t]
\caption{Quantitative results on asymptotic performance. Our proposed method achieved comparable performance with only half the number of interactions with the environment, and achieves superior performance for $1$M interactions.}
\label{table_tasks}
\begin{center}
\scriptsize
\begin{sc}
\begin{tabular}{lccccc}
\toprule
Method/Task & Hopper-v2 & Walker2d-v2 & HalfCheetah-v2 & Ant-v2 & Humanoid-v2 \\
\midrule
TD3  & $2843 \pm 197$ & $3842 \pm 239$ & $10314\pm 93$ & $4868\pm 388$ & $4855 \pm 263$\\
SAC & $2158\pm 388$& $4154 \pm 333$ & $8735 \pm 170$ & $3051\pm 469$ & $\mathbf{5012 \pm 478}$ \\
OAC& $2983 \pm 317$ & $3075 \pm 183$ & $4497 \pm 296$& $4497 \pm297$& $4624 \pm 351$\\
Ours-$0.5$M & $3048 \pm 307$ & $3932\pm 125$& $10290\pm129$ &$4304\pm 260$ & $4175 \pm 302$\\
Ours & $\mathbf{3268 \pm 234}$  & $\mathbf{4027 \pm 28}$  & $\mathbf{10992 \pm 126}$  & $\mathbf{5006\pm 135}$ & $4881 \pm 164$ \\
%Improvement Over TD3/SAC & $\uparrow31\%/$ & $(\downarrow2\%)$ & $(\uparrow15\%)$ & $(\uparrow21\%)$ &$(\downarrow1\%)$
Improv. Over TD3/SAC & $\uparrow15\%/\uparrow51\%$ & $\uparrow5\%/\downarrow3\%$ & $\uparrow7\%/\uparrow26\%$ & $\uparrow3\%/\uparrow64\%$ &$\uparrow1\%/\downarrow3\%$ \\
\bottomrule
\end{tabular}
\end{sc}
\end{center}
\end{table}
\begin{figure}[t]
\centering
\begin{minipage}[htbp]{0.33\linewidth}
	\centering
	\includegraphics[width=1\linewidth]{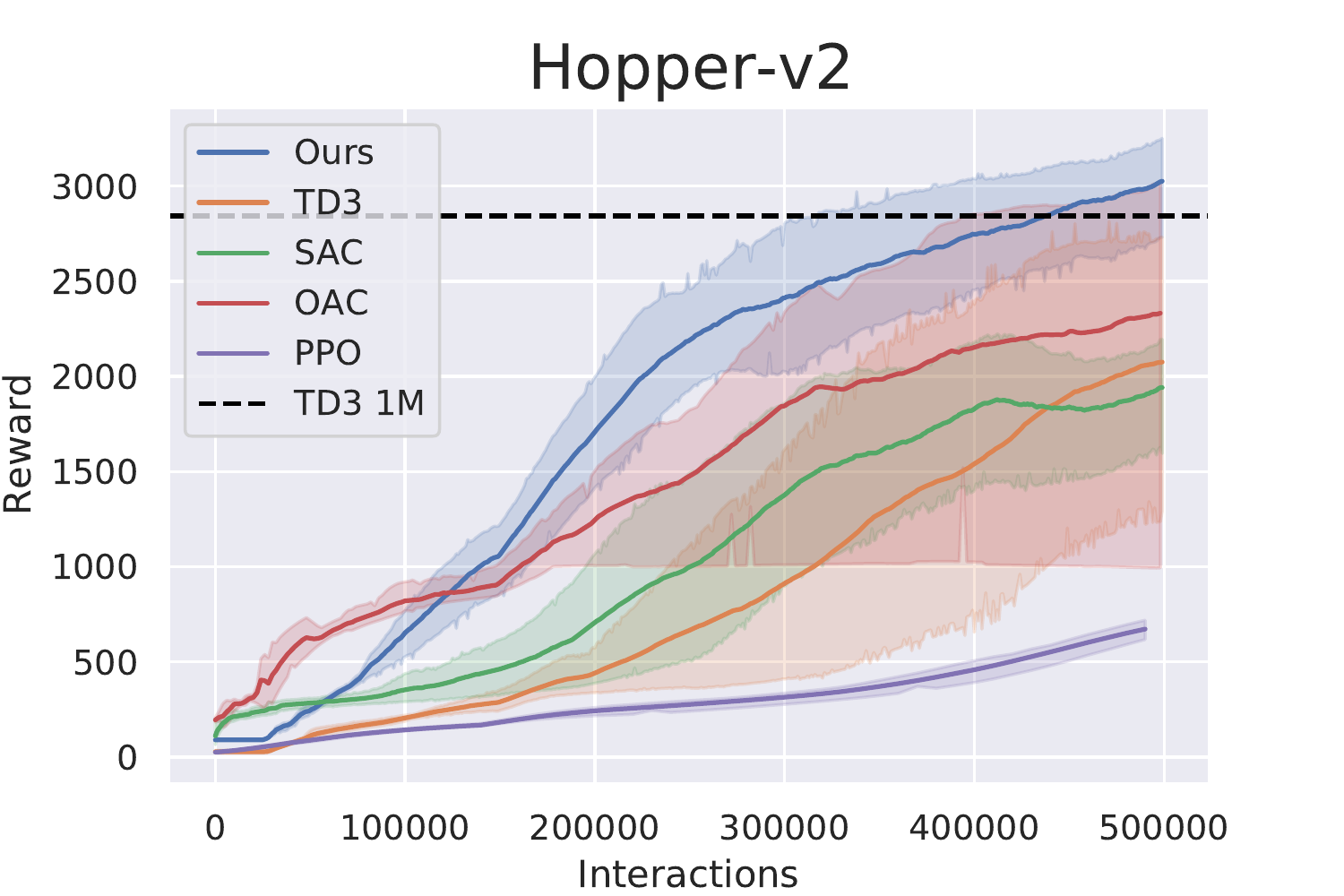}
\end{minipage}%
\begin{minipage}[htbp]{0.33\linewidth}
	\centering
	\includegraphics[width=1\linewidth]{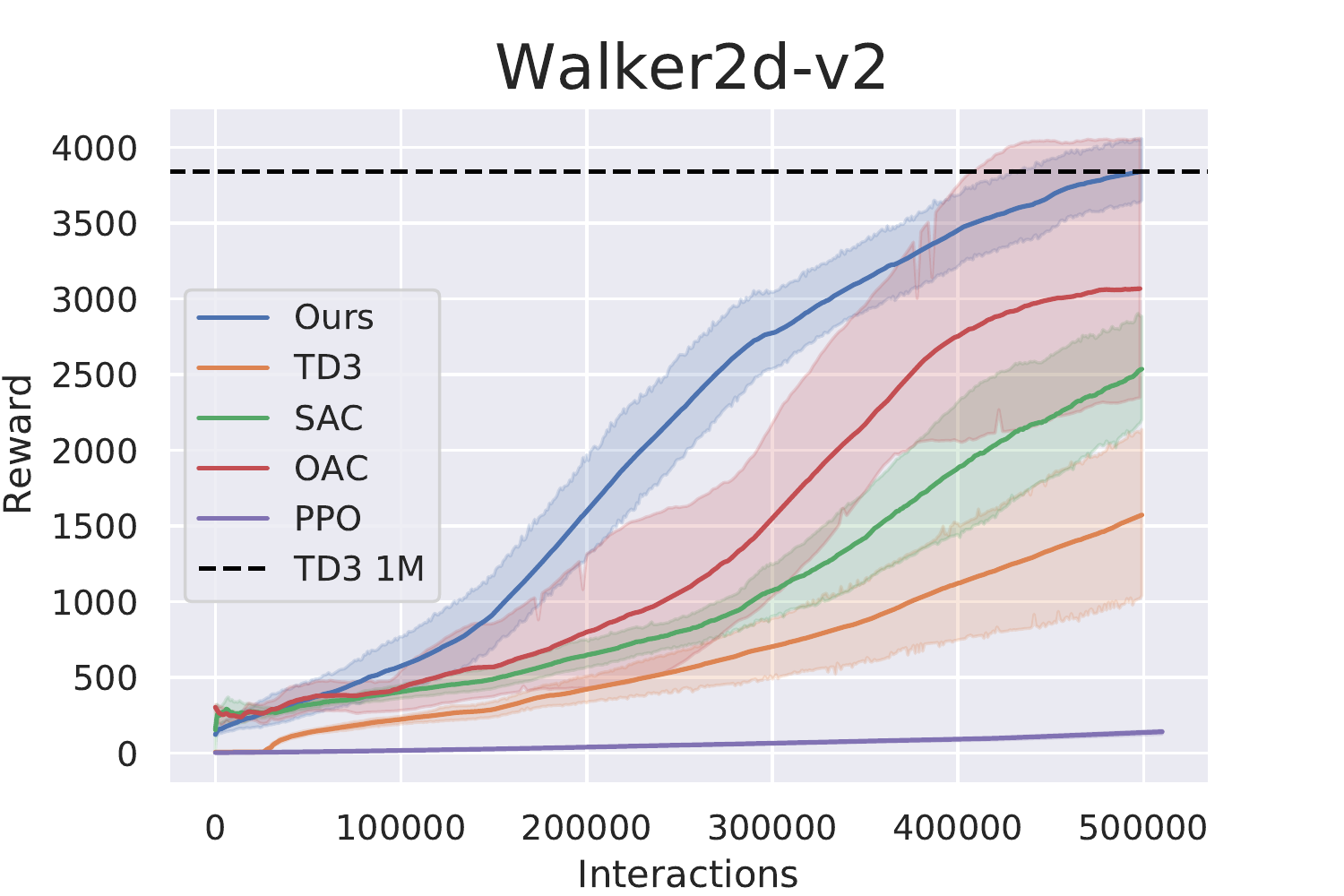}
\end{minipage}%
\begin{minipage}[htbp]{0.33\linewidth}
	\centering
	\includegraphics[width=1\linewidth]{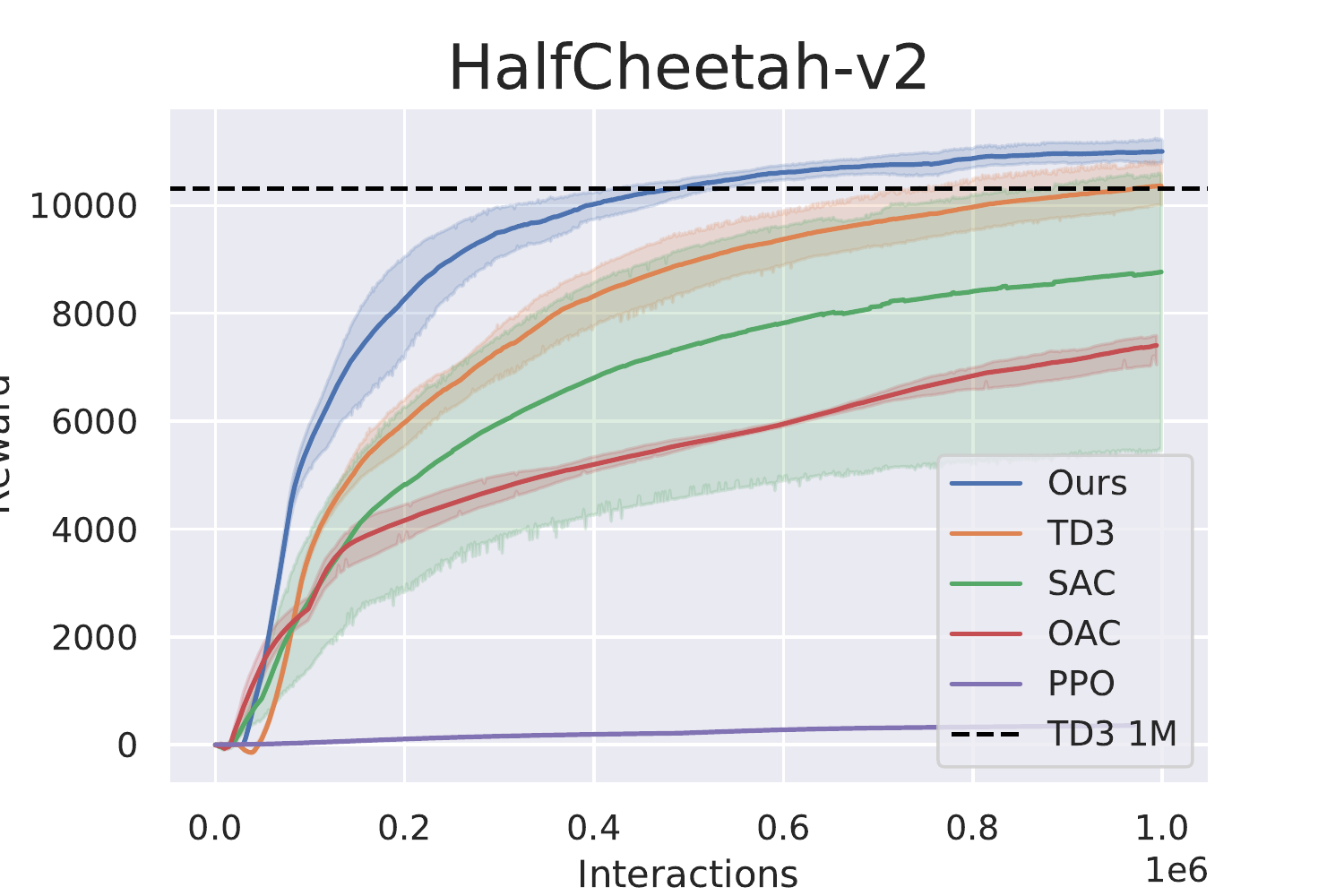}
\end{minipage}\\%
\begin{minipage}[htbp]{0.33\linewidth}
	\centering
	\includegraphics[width=1\linewidth]{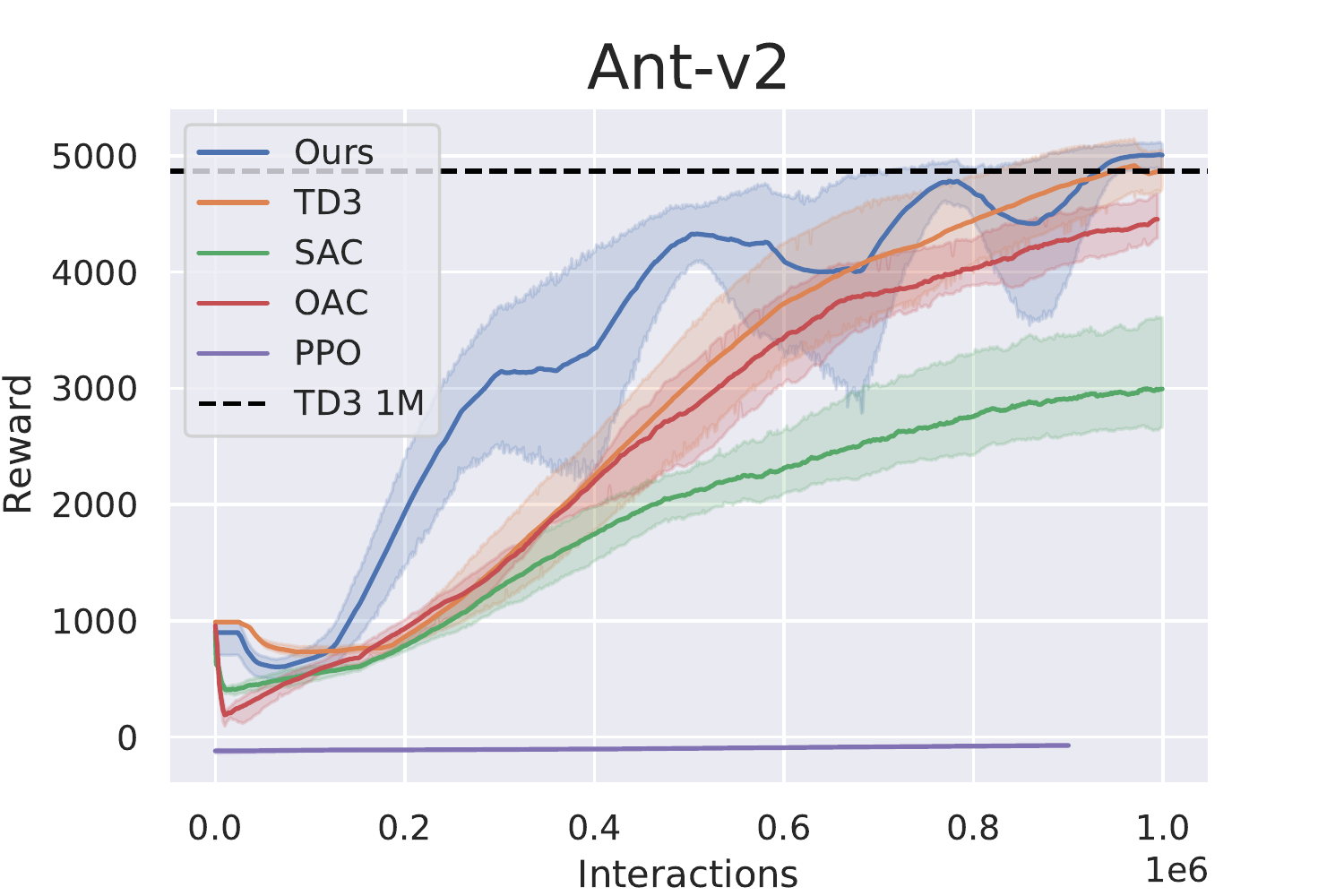}
\end{minipage}%
\begin{minipage}[htbp]{0.33\linewidth}
	\centering
	\includegraphics[width=1\linewidth]{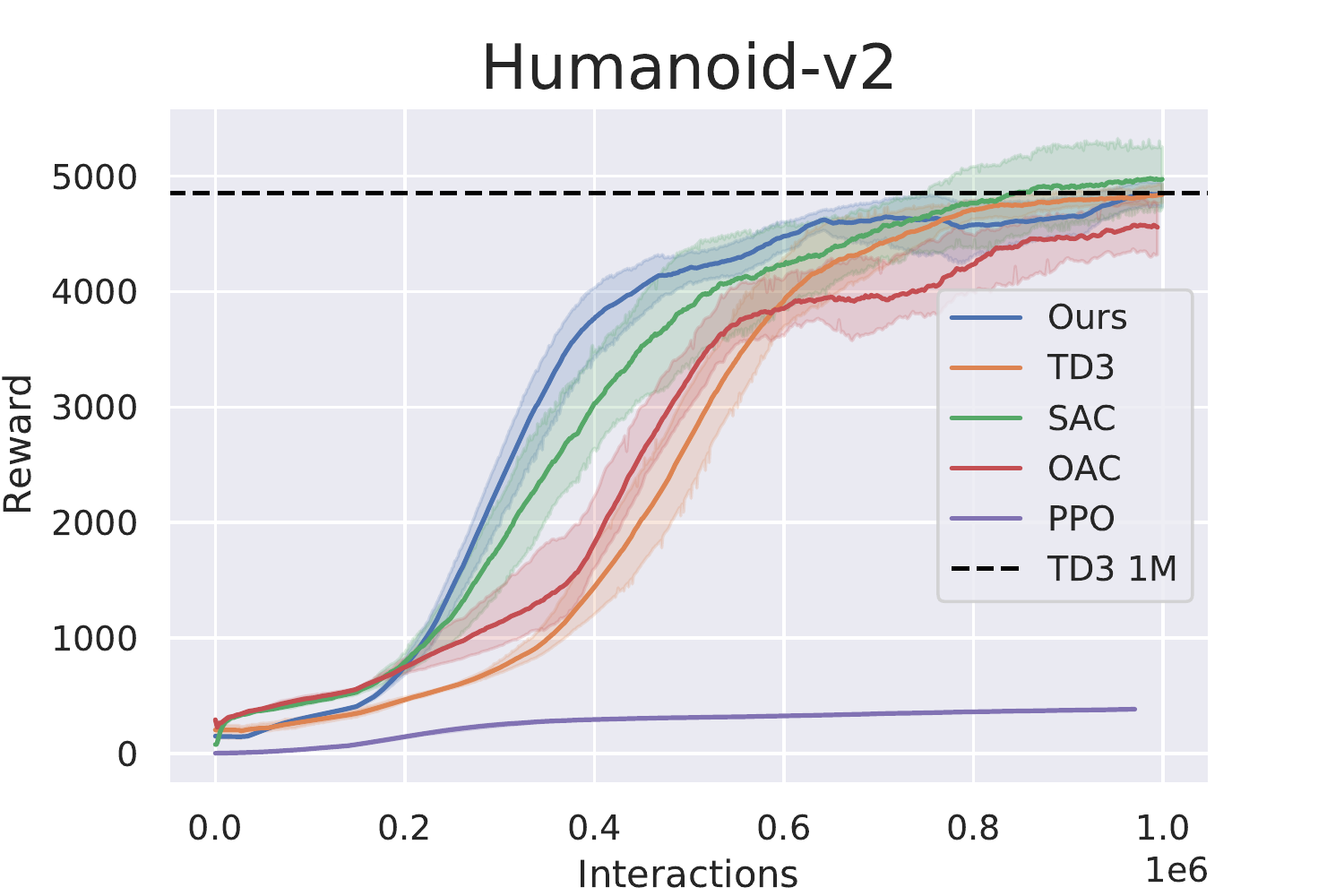}
\end{minipage}%
\begin{minipage}[htbp]{0.33\linewidth}
	\centering
	\includegraphics[width=1\linewidth]{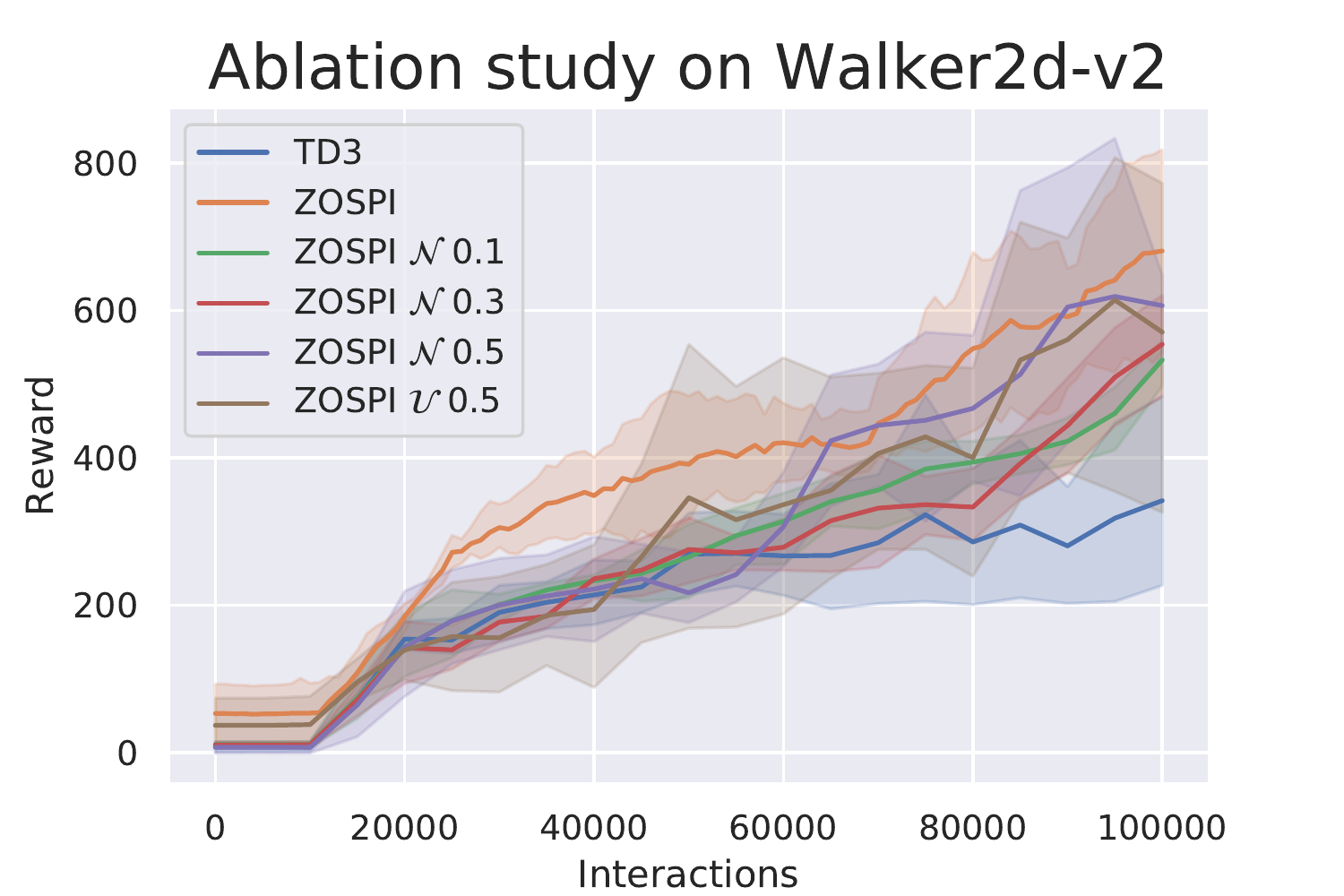}
\end{minipage}%

\caption{Experimental results on the MuJoCo locomotion tasks. The shaded region represents half a standard deviation. The dashed lines indicate asymptotic performance of TD3 after $1$M interactions. ZOSPI is able to reach on-par performance within much less interactions. In all main experiments the results reported are collected from $10$ random seeds and in ablations studies we use $5$ random seeds. Curves are smoothed uniformly for visual clarity.
}
\label{mujoco_results}
\end{figure}
We evaluate ZOSPI on the Gym locomotion tasks based on the MuJoCo engine~\cite{1606.01540,TodorovET12}. %Concretely, we test ZOSPI on five locomotion environments, namely the Hopper-v2 (11-dim observations and 3-dim actions), Walker2d-v2 (17-dim observations and 6-dim actions), HalfCheetah-v2 (17-dim observations and 6-dim actions), Ant-v2 (111-dim observations and 8-dim actions) and Humanoid-v2 (376-dim observations and 17-dim actions). We compare results of different methods within $300,000$ environment interactions to demonstrate the high learning efficiency of ZOSPI. We include TD3 and SAC, respectively a deterministic and a stochastic SOTA policy gradient methods in the comparison. The results of TD3 are obtained by running author-released codes and the results of SAC are directly extracted from the training logs released by the authors.
The five locomotion tasks are Hopper-v2, Walker2d-v2, HalfCheetah-v2, Ant-v2, and Humanoid-v2. We compare our method with TD3 and SAC, the deterministic and the stochastic SOTA policy gradient methods. We also include PPO and OAC~\cite{ciosek2019better} to better show the learning efficiency of ZOSPI. We compare different methods within $0.5$M environment interactions in the two easy tasks to demonstrate the high learning efficiency of ZOSPI, and present results during $1$M interactions for the other tasks. % It is worth noting that in such $0.3$M interactions ZOSPI achieves on-par performance with the $1$M-timestep results reported in~\cite{haarnoja2018softappli} and therefore we believe it is enough to demonstrate the sample efficiency of our self-supervised innovative approach. 
The results of TD3 and OAC are obtained by running the code released by the authors, the results of PPO are obtained through the high quality open-source implementation~\cite{baselines}, and the results of SAC are extracted from the training logs of~\cite{haarnoja2018soft}. 

The learning curves of those tasks are shown in Figure~\ref{mujoco_results}. It is worth noting that in all of the results we reported, only $50$ actions are sampled and it is sufficient to learn well-performing policies. With a high sampling efficiency, ZOSPI works well in the challenging environments that have high-dimensional action spaces such as Ant-v2 and Humanoid-v2. In all the tasks the sample efficiency is consistently improved over TD3, which is the DPG counterpart of ZOSPI. 
While a total of $50$ sampled actions should be very sparse in the high dimensional space, we attribute the success of ZOSPI to the generality of the policy network as well as the sparsity of meaningful actions. For example, even in the tasks that have high dimensional action spaces, only limited dimensions of the action are crucial.

We provide quantitative comparison in Table~\ref{table_tasks}. We report both the averaged score and the standard deviation of different methods after $1$M step of interactions with the environment. Moreover, we provide the quantitative results of ZOSPI with $0.5$M interactions during training to emphasize on its high sample efficiency. In $4$ out of the $5$ environments (except Humanoid), ZOSPI is able to achieve on-par performance compared to TD3 and SAC, with only half number of interactions. The last line of Table~\ref{table_tasks} reports the improvement of ZOSPI over TD3 and SAC respectively. All results are collected by experimenting with $10$ random seeds.

The last plot in Figure~\ref{mujoco_results} shows the ablation study on the sampling range in ZOSPI, where a sampling method based on a zero-mean Gaussian is applied and we gradually increase its variance from $0.1$ to $0.5$. We also evaluate the uniform sampling method with radius of $0.5$, which is denoted as $\mathcal{U}$ $0.5$ in the Figure.
The results suggest that zeroth-order optimization with local sampling performs similarly to the policy gradient method, and increasing the sampling range can effectively improve the performance.
% \subsection{Learning Multi-Modal Policies}
% DPO, SAC, SQL, Imitation + MDN, World Models + MDN
\subsection{Ablation Study}
In this section, we conduct a series of experiments as the ablation study to evaluate different components of ZOSPI. Specifically, we investigate the performance of ZOSPI under different number of samples, with or without perturbation networks, and Number of Diracs (NoD) used in the policy network.

\paragraph{Number of Samples Used in ZOSPI} The first column of Figure~\ref{ablation_results} shows the performance of ZOSPI when using different number of samples. In both environments, more samples lead to better performance, but in the easier task of Hopper, there are only tiny differences, while in the task of Walker2d, using more than $50$ samples noticeably improves the performance. Trading off the improvements and the computational costs when using more samples, we choose to use $50$ samples for the experiments we reported in the previous section. %In practice, a larger number of sample size can be used to pursue higher sample efficiency with the increasing of computational cost.

\begin{figure}[t]
\centering
\begin{minipage}[b]{0.32\linewidth}
\label{four_way_maze}
\includegraphics[width=1.0\linewidth]{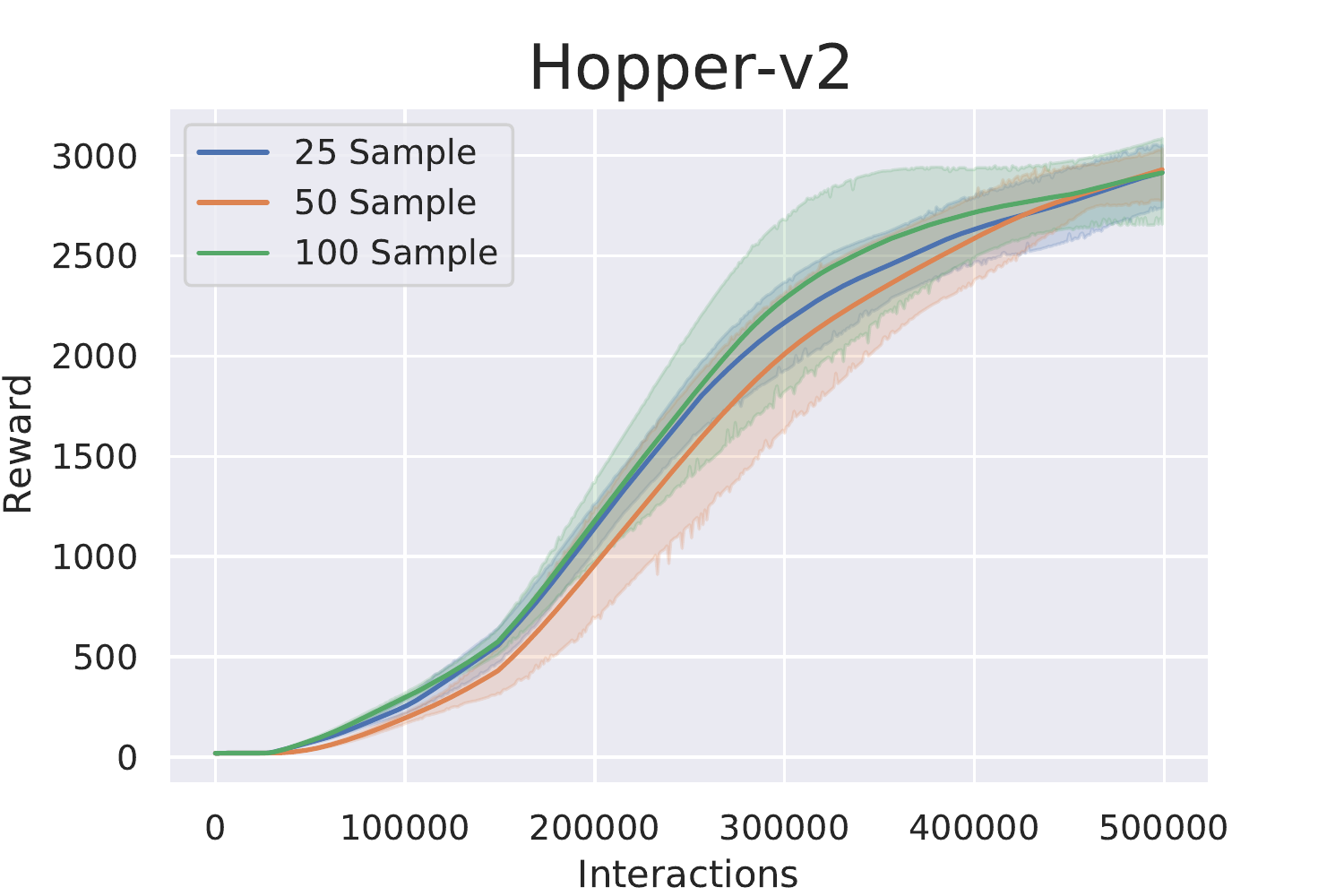}
\end{minipage}
\begin{minipage}[b]{0.32\linewidth}
\includegraphics[width=1.0\linewidth]{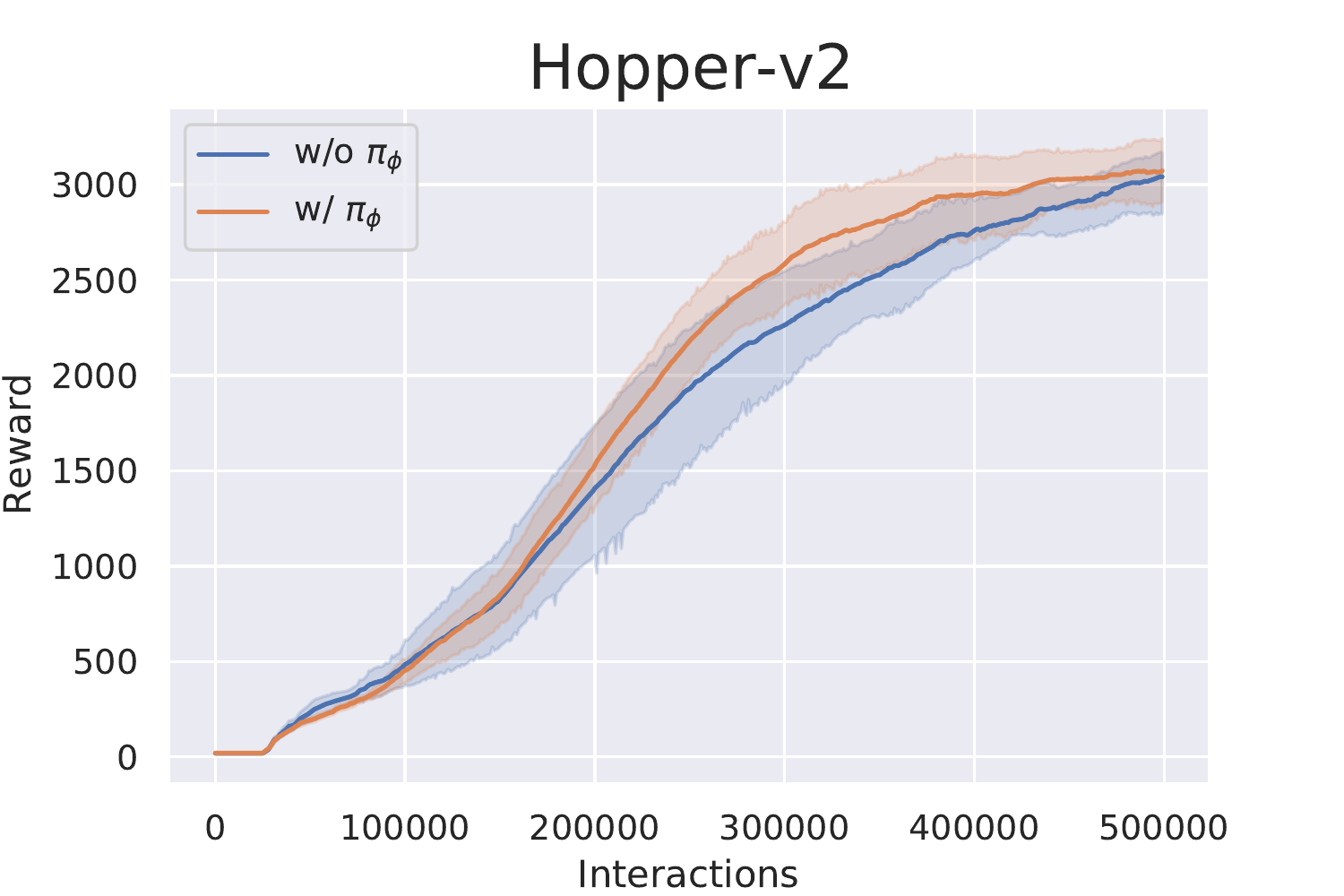}
\end{minipage}
\begin{minipage}[b]{0.32\linewidth}
\label{toy_curve}
\includegraphics[width=1.0\linewidth]{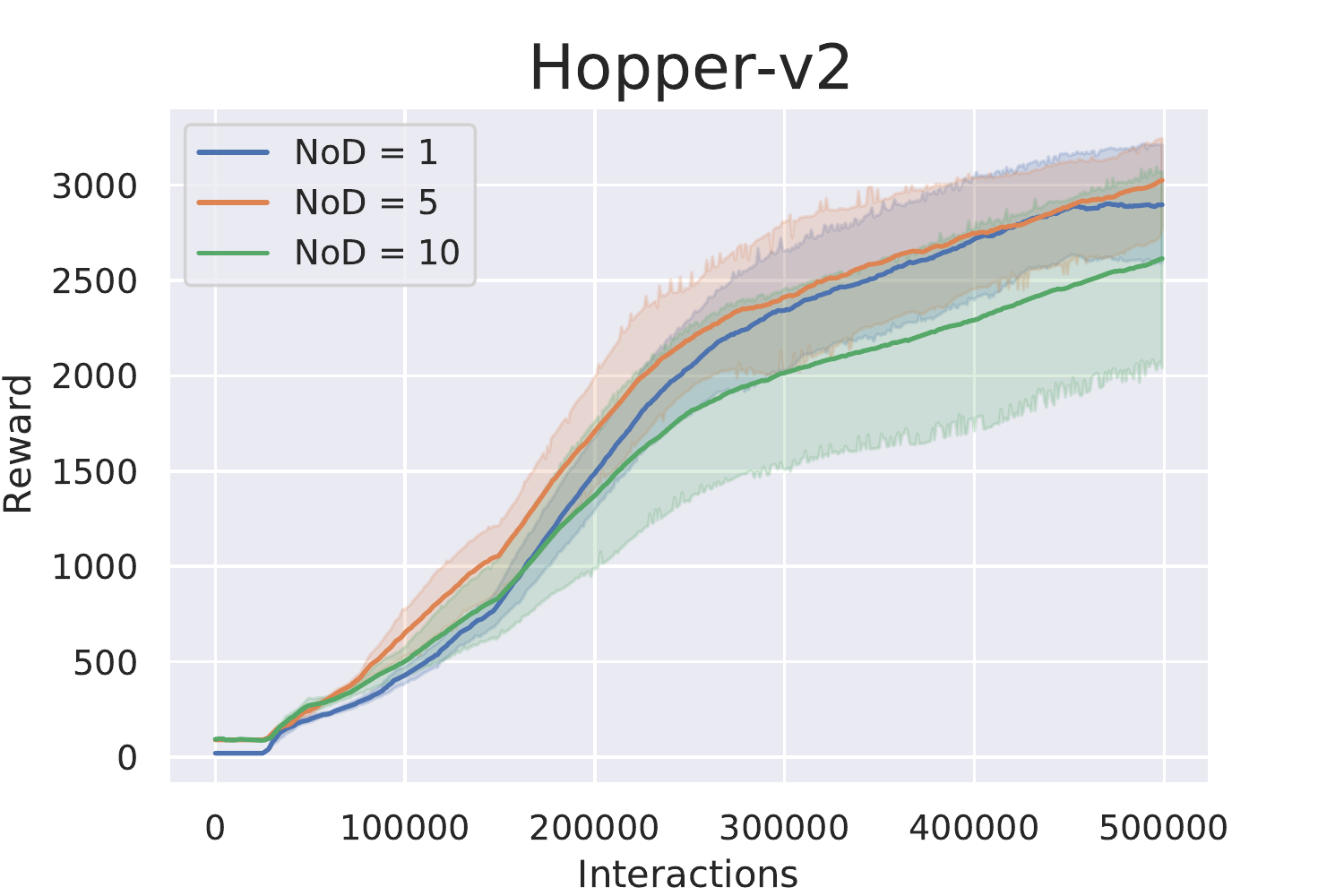}%\vspace{4pt}
\end{minipage}\\
\subfigure[Number of Samples]{
\begin{minipage}[b]{0.32\linewidth}
\label{vis_toy_start}
\includegraphics[width=1.0\linewidth]{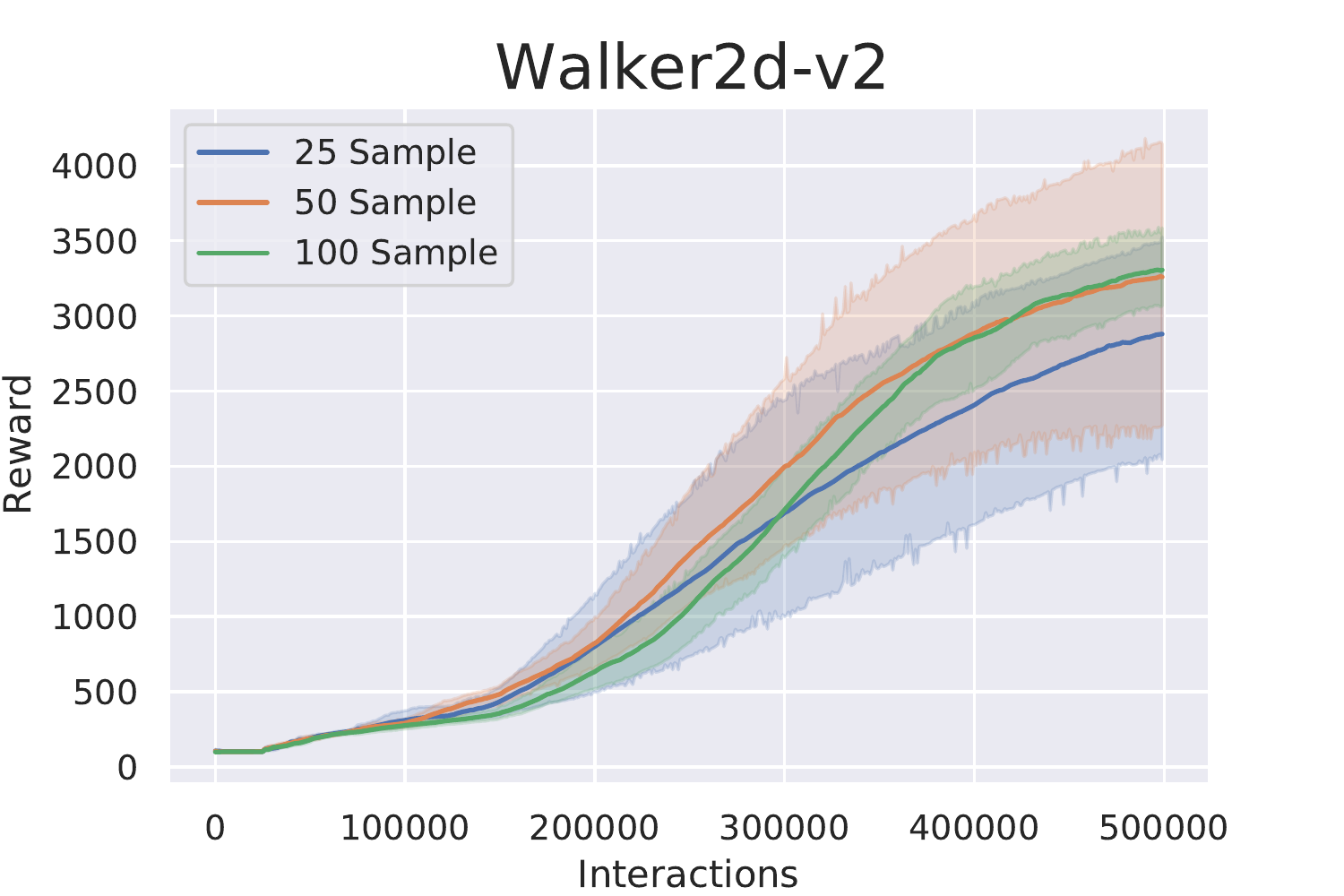}%\vspace{4pt}
\end{minipage}}
\subfigure[Perturbation Network]{
\begin{minipage}[b]{0.32\linewidth}
\includegraphics[width=1.0\linewidth]{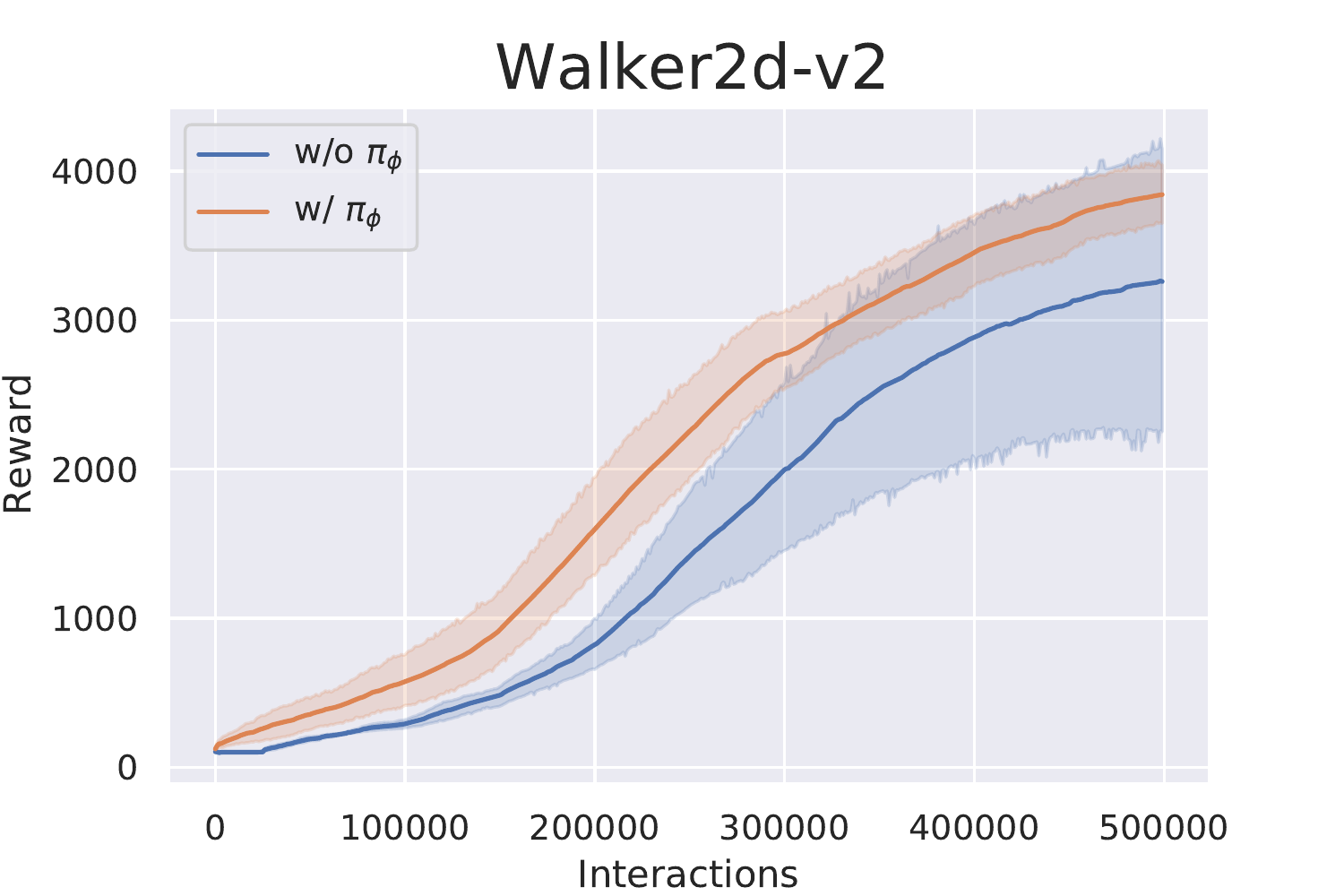}%\vspace{4pt}
\end{minipage}}
\subfigure[Number of Gaussians]{
\begin{minipage}[b]{0.32\linewidth}
\includegraphics[width=1.0\linewidth]{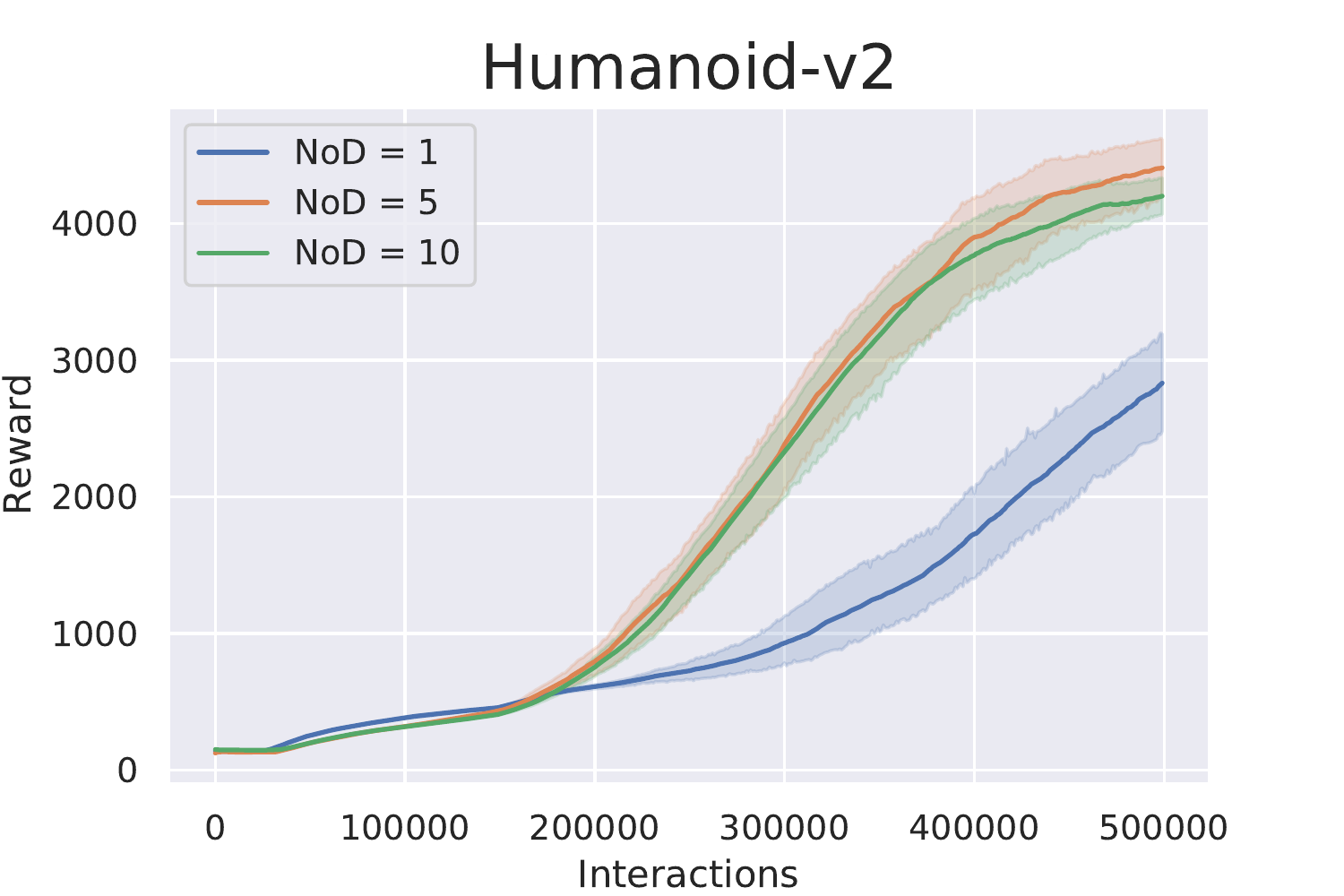}%\vspace{4pt}
\end{minipage}}
\caption{Ablation study. The first column (a) shows the performance of our proposed method with different number of samples; the second column (b) shows the performance difference when ZOSPI is work with or without the perturbation network; the last column (c) shows the performance difference between using different number of Diracs when combining MDNs with ZOSPI.}
\label{ablation_results}
\end{figure}

\paragraph{Perturbation Network} The second column of Figure~\ref{ablation_results} shows the ablation studies on the perturbation network. While in the simpler control task of Hopper, the perturbation network helps improve the learning efficiency only a little bit, in more complex task the perturbation network is crucial for efficient learning. This result demonstrate the effectiveness of the combination of global and local information in exploiting of the learned value function. In all of our main experiments, we use a perturbation network with a perturbation range of $0.05$.

\paragraph{The Application of MDNs}
Our ablation studies on the usage of MDNs are shown in the last column of Figure~\ref{ablation_results}. In the easier task of Hopper, using a larger number of NoD in ZOSPI with MDN only improves a little but introduce several times of extra computation expense. However, such multi-modal policies clearly improves the performance in the complex tasks like Humanoid. In our experiments, we find larger NoD benefits the learning for the Walker2d and Humanoid tasks, and do not improve performance in the other three environments. More implementation details of ZOSPI with MDN can be found in Appendix~\ref{appd_zospi_mdn}.

\section{Conclusion}
In this work, we propose the method of Zeroth-Order Supervised Policy Improvement (ZOSPI) as an alternative approach of policy gradient algorithms for continuous control. ZOSPI improves the learning efficiency of previous policy gradient learning by exploiting the learned $Q$ functions globally and locally through sampling the action space.
Different from previous policy gradient methods, the policy optimization of ZOSPI is based on supervised learning so that the learning of actor can be implemented with regression. Such a property enables  the potential extensions such as multi-modal policies, which can be seamlessly cooperated with ZOSPI. We evaluate ZOSPI on five locomotion tasks, where the proposed method remarkably improves the performance in terms of both sample efficiency and asymptotic performance compared to previous policy gradient methods.

% We also propose to combine ZOSPI with a bootstrapped estimation of the $Q$ function for further improvements.
% We provide the theoretic analysis to validate that the proposed ZOSPI can be a sample efficient method. 

%Specifically, ZOSPI is able to perform pretty well with only $50$ sampled actions.



\newpage
\bibliography{iclr2021_conference}

\begin{thebibliography}{8}
\providecommand{\natexlab}[1]{#1}
\providecommand{\url}[1]{\texttt{#1}}
\expandafter\ifx\csname urlstyle\endcsname\relax
  \providecommand{\doi}[1]{doi: #1}\else
  \providecommand{\doi}{doi: \begingroup \urlstyle{rm}\Url}\fi

\bibitem[Author(2021)]{anonymous}
Author, N.~N.
\newblock Suppressed for anonymity, 2021.

\bibitem[Duda et~al.(2000)Duda, Hart, and Stork]{DudaHart2nd}
Duda, R.~O., Hart, P.~E., and Stork, D.~G.
\newblock \emph{Pattern Classification}.
\newblock John Wiley and Sons, 2nd edition, 2000.

\bibitem[Kearns(1989)]{kearns89}
Kearns, M.~J.
\newblock \emph{Computational Complexity of Machine Learning}.
\newblock PhD thesis, Department of Computer Science, Harvard University, 1989.

\bibitem[Langley(2000)]{langley00}
Langley, P.
\newblock Crafting papers on machine learning.
\newblock In Langley, P. (ed.), \emph{Proceedings of the 17th International
  Conference on Machine Learning (ICML 2000)}, pp.\  1207--1216, Stanford, CA,
  2000. Morgan Kaufmann.

\bibitem[Michalski et~al.(1983)Michalski, Carbonell, and
  Mitchell]{MachineLearningI}
Michalski, R.~S., Carbonell, J.~G., and Mitchell, T.~M. (eds.).
\newblock \emph{Machine Learning: An Artificial Intelligence Approach, Vol. I}.
\newblock Tioga, Palo Alto, CA, 1983.

\bibitem[Mitchell(1980)]{mitchell80}
Mitchell, T.~M.
\newblock The need for biases in learning generalizations.
\newblock Technical report, Computer Science Department, Rutgers University,
  New Brunswick, MA, 1980.

\bibitem[Newell \& Rosenbloom(1981)Newell and Rosenbloom]{Newell81}
Newell, A. and Rosenbloom, P.~S.
\newblock Mechanisms of skill acquisition and the law of practice.
\newblock In Anderson, J.~R. (ed.), \emph{Cognitive Skills and Their
  Acquisition}, chapter~1, pp.\  1--51. Lawrence Erlbaum Associates, Inc.,
  Hillsdale, NJ, 1981.

\bibitem[Samuel(1959)]{Samuel59}
Samuel, A.~L.
\newblock Some studies in machine learning using the game of checkers.
\newblock \emph{IBM Journal of Research and Development}, 3\penalty0
  (3):\penalty0 211--229, 1959.

\end{thebibliography}


\begin{thebibliography}{10}

\bibitem{mnih2015human}
V.~Mnih, K.~Kavukcuoglu, D.~Silver, A.~A. Rusu, J.~Veness, M.~G. Bellemare,
  A.~Graves, M.~Riedmiller, A.~K. Fidjeland, G.~Ostrovski, {\em et~al.},
  ``Human-level control through deep reinforcement learning,'' {\em Nature},
  vol.~518, no.~7540, pp.~529--533, 2015.

\bibitem{vinyals2019grandmaster}
O.~Vinyals, I.~Babuschkin, W.~M. Czarnecki, M.~Mathieu, A.~Dudzik, J.~Chung,
  D.~H. Choi, R.~Powell, T.~Ewalds, P.~Georgiev, {\em et~al.}, ``Grandmaster
  level in starcraft ii using multi-agent reinforcement learning,'' {\em
  Nature}, vol.~575, no.~7782, pp.~350--354, 2019.

\bibitem{pachockiopenai}
J.~Pachocki, G.~Brockman, J.~Raiman, S.~Zhang, H.~Pond{\'e}, J.~Tang,
  F.~Wolski, C.~Dennison, R.~Jozefowicz, P.~Debiak, {\em et~al.}, ``Openai
  five, 2018,'' {\em URL https://blog. openai. com/openai-five}.

\bibitem{degris2012off}
T.~Degris, M.~White, and R.~S. Sutton, ``Off-policy actor-critic,'' {\em arXiv
  preprint arXiv:1205.4839}, 2012.

\bibitem{gu2016q}
S.~Gu, T.~Lillicrap, Z.~Ghahramani, R.~E. Turner, and S.~Levine, ``Q-prop:
  Sample-efficient policy gradient with an off-policy critic,'' {\em arXiv
  preprint arXiv:1611.02247}, 2016.

\bibitem{wang2016sample}
Z.~Wang, V.~Bapst, N.~Heess, V.~Mnih, R.~Munos, K.~Kavukcuoglu, and
  N.~de~Freitas, ``Sample efficient actor-critic with experience replay,'' {\em
  arXiv preprint arXiv:1611.01224}, 2016.

\bibitem{lillicrap2015continuous}
T.~P. Lillicrap, J.~J. Hunt, A.~Pritzel, N.~Heess, T.~Erez, Y.~Tassa,
  D.~Silver, and D.~Wierstra, ``Continuous control with deep reinforcement
  learning,'' {\em arXiv preprint arXiv:1509.02971}, 2015.

\bibitem{fujimoto2018addressing}
S.~Fujimoto, H.~Van~Hoof, and D.~Meger, ``Addressing function approximation
  error in actor-critic methods,'' {\em arXiv preprint arXiv:1802.09477}, 2018.

\bibitem{schulman2015trust}
J.~Schulman, S.~Levine, P.~Abbeel, M.~Jordan, and P.~Moritz, ``Trust region
  policy optimization,'' in {\em International conference on machine learning},
  pp.~1889--1897, 2015.

\bibitem{schulman2017proximal}
J.~Schulman, F.~Wolski, P.~Dhariwal, A.~Radford, and O.~Klimov, ``Proximal
  policy optimization algorithms,'' {\em arXiv preprint arXiv:1707.06347},
  2017.

\bibitem{haarnoja2018soft}
T.~Haarnoja, A.~Zhou, P.~Abbeel, and S.~Levine, ``Soft actor-critic: Off-policy
  maximum entropy deep reinforcement learning with a stochastic actor,'' {\em
  arXiv preprint arXiv:1801.01290}, 2018.

\bibitem{haarnoja2017reinforcement}
T.~Haarnoja, H.~Tang, P.~Abbeel, and S.~Levine, ``Reinforcement learning with
  deep energy-based policies,'' in {\em Proceedings of the 34th International
  Conference on Machine Learning-Volume 70}, pp.~1352--1361, JMLR. org, 2017.

\bibitem{zhang2019generalized}
S.~Zhang, W.~Boehmer, and S.~Whiteson, ``Generalized off-policy actor-critic,''
  {\em arXiv preprint arXiv:1903.11329}, 2019.

\bibitem{ciosek2019better}
K.~Ciosek, Q.~Vuong, R.~Loftin, and K.~Hofmann, ``Better exploration with
  optimistic actor critic,'' in {\em Advances in Neural Information Processing
  Systems}, pp.~1785--1796, 2019.

\bibitem{brafman2002r}
R.~I. Brafman and M.~Tennenholtz, ``R-max-a general polynomial time algorithm
  for near-optimal reinforcement learning,'' {\em Journal of Machine Learning
  Research}, vol.~3, no.~Oct, pp.~213--231, 2002.

\bibitem{tessler2019distributional}
C.~Tessler, G.~Tennenholtz, and S.~Mannor, ``Distributional policy
  optimization: An alternative approach for continuous control,'' {\em arXiv
  preprint arXiv:1905.09855}, 2019.

\bibitem{silver2014deterministic}
D.~Silver, G.~Lever, N.~Heess, T.~Degris, D.~Wierstra, and M.~Riedmiller,
  ``Deterministic policy gradient algorithms,'' in {\em ICML}, 2014.

\bibitem{salimans2017evolution}
T.~Salimans, J.~Ho, X.~Chen, S.~Sidor, and I.~Sutskever, ``Evolution strategies
  as a scalable alternative to reinforcement learning,'' {\em arXiv preprint
  arXiv:1703.03864}, 2017.

\bibitem{conti2018improving}
E.~Conti, V.~Madhavan, F.~P. Such, J.~Lehman, K.~Stanley, and J.~Clune,
  ``Improving exploration in evolution strategies for deep reinforcement
  learning via a population of novelty-seeking agents,'' in {\em Advances in
  Neural Information Processing Systems}, pp.~5027--5038, 2018.

\bibitem{mania2018simple}
H.~Mania, A.~Guy, and B.~Recht, ``Simple random search provides a competitive
  approach to reinforcement learning,'' {\em arXiv preprint arXiv:1803.07055},
  2018.

\bibitem{williams1992simple}
R.~J. Williams, ``Simple statistical gradient-following algorithms for
  connectionist reinforcement learning,'' {\em Machine learning}, vol.~8,
  no.~3-4, pp.~229--256, 1992.

\bibitem{sutton1998reinforcement}
R.~S. Sutton and A.~G. Barto, ``Reinforcement learning: An introduction,''
  1998.

\bibitem{sutton2000policy}
R.~S. Sutton, D.~A. McAllester, S.~P. Singh, and Y.~Mansour, ``Policy gradient
  methods for reinforcement learning with function approximation,'' in {\em
  Advances in neural information processing systems}, pp.~1057--1063, 2000.

\bibitem{sun2019policy}
H.~Sun, Z.~Li, X.~Liu, B.~Zhou, and D.~Lin, ``Policy continuation with
  hindsight inverse dynamics,'' in {\em Advances in Neural Information
  Processing Systems}, pp.~10265--10275, 2019.

\bibitem{ghosh2019learning}
D.~Ghosh, A.~Gupta, J.~Fu, A.~Reddy, C.~Devine, B.~Eysenbach, and S.~Levine,
  ``Learning to reach goals without reinforcement learning,'' {\em arXiv
  preprint arXiv:1912.06088}, 2019.

\bibitem{wang2018exponentially}
Q.~Wang, J.~Xiong, L.~Han, H.~Liu, T.~Zhang, {\em et~al.}, ``Exponentially
  weighted imitation learning for batched historical data,'' in {\em Advances
  in Neural Information Processing Systems}, pp.~6288--6297, 2018.

\bibitem{zhang2019policy}
C.~Zhang, Y.~Li, and J.~Li, ``Policy search by target distribution learning for
  continuous control,'' {\em arXiv preprint arXiv:1905.11041}, 2019.

\bibitem{abdolmaleki2018relative}
A.~Abdolmaleki, J.~T. Springenberg, J.~Degrave, S.~Bohez, Y.~Tassa, D.~Belov,
  N.~Heess, and M.~Riedmiller, ``Relative entropy regularized policy
  iteration,'' {\em arXiv preprint arXiv:1812.02256}, 2018.

\bibitem{song2019v}
H.~F. Song, A.~Abdolmaleki, J.~T. Springenberg, A.~Clark, H.~Soyer, J.~W. Rae,
  S.~Noury, A.~Ahuja, S.~Liu, D.~Tirumala, {\em et~al.}, ``V-mpo: On-policy
  maximum a posteriori policy optimization for discrete and continuous
  control,'' {\em arXiv preprint arXiv:1909.12238}, 2019.

\bibitem{lim2018actor}
S.~Lim, A.~Joseph, L.~Le, Y.~Pan, and M.~White, ``Actor-expert: A framework for
  using q-learning in continuous action spaces,'' {\em arXiv preprint
  arXiv:1810.09103}, 2018.

\bibitem{simmons2019q}
R.~Simmons-Edler, B.~Eisner, E.~Mitchell, S.~Seung, and D.~Lee, ``Q-learning
  for continuous actions with cross-entropy guided policies,'' {\em arXiv
  preprint arXiv:1903.10605}, 2019.

\bibitem{wang2017stochastic}
Y.~Wang, S.~Du, S.~Balakrishnan, and A.~Singh, ``Stochastic zeroth-order
  optimization in high dimensions,'' {\em arXiv preprint arXiv:1710.10551},
  2017.

\bibitem{golovin2019gradientless}
D.~Golovin, J.~Karro, G.~Kochanski, C.~Lee, X.~Song, {\em et~al.},
  ``Gradientless descent: High-dimensional zeroth-order optimization,'' {\em
  arXiv preprint arXiv:1911.06317}, 2019.

\bibitem{vlatakis2019efficiently}
E.-V. Vlatakis-Gkaragkounis, L.~Flokas, and G.~Piliouras, ``Efficiently
  avoiding saddle points with zero order methods: No gradients required,'' in
  {\em Advances in Neural Information Processing Systems}, pp.~10066--10077,
  2019.

\bibitem{bai2020escaping}
Q.~Bai, M.~Agarwal, and V.~Aggarwal, ``Escaping saddle points for zeroth-order
  non-convex optimization using estimated gradient descent,'' in {\em 2020 54th
  Annual Conference on Information Sciences and Systems (CISS)}, pp.~1--6,
  IEEE, 2020.

\bibitem{usunier2016episodic}
N.~Usunier, G.~Synnaeve, Z.~Lin, and S.~Chintala, ``Episodic exploration for
  deep deterministic policies for starcraft micromanagement,'' 2016.

\bibitem{konda2000actor}
V.~R. Konda and J.~N. Tsitsiklis, ``Actor-critic algorithms,'' in {\em Advances
  in neural information processing systems}, pp.~1008--1014, 2000.

\bibitem{peters2008natural}
J.~Peters and S.~Schaal, ``Natural actor-critic,'' {\em Neurocomputing},
  vol.~71, no.~7-9, pp.~1180--1190, 2008.

\bibitem{bishop1994mixture}
C.~M. Bishop, ``Mixture density networks,'' 1994.

\bibitem{fujimoto2018off}
S.~Fujimoto, D.~Meger, and D.~Precup, ``Off-policy deep reinforcement learning
  without exploration,'' {\em arXiv preprint arXiv:1812.02900}, 2018.

\bibitem{1606.01540}
G.~Brockman, V.~Cheung, L.~Pettersson, J.~Schneider, J.~Schulman, J.~Tang, and
  W.~Zaremba, ``Openai gym,'' 2016.

\bibitem{TodorovET12}
E.~Todorov, T.~Erez, and Y.~Tassa, ``Mujoco: A physics engine for model-based
  control.,'' in {\em IROS}, pp.~5026--5033, IEEE, 2012.

\bibitem{baselines}
P.~Dhariwal, C.~Hesse, O.~Klimov, A.~Nichol, M.~Plappert, A.~Radford,
  J.~Schulman, S.~Sidor, Y.~Wu, and P.~Zhokhov, ``Openai baselines.''
  \url{https://github.com/openai/baselines}, 2017.

\bibitem{jaksch2010near}
T.~Jaksch, R.~Ortner, and P.~Auer, ``Near-optimal regret bounds for
  reinforcement learning,'' {\em Journal of Machine Learning Research},
  vol.~11, no.~Apr, pp.~1563--1600, 2010.

\bibitem{azar2017minimax}
M.~G. Azar, I.~Osband, and R.~Munos, ``Minimax regret bounds for reinforcement
  learning,'' in {\em Proceedings of the 34th International Conference on
  Machine Learning-Volume 70}, pp.~263--272, JMLR. org, 2017.

\bibitem{jin2018q}
C.~Jin, Z.~Allen-Zhu, S.~Bubeck, and M.~I. Jordan, ``Is q-learning provably
  efficient?,'' in {\em Advances in Neural Information Processing Systems},
  pp.~4863--4873, 2018.

\bibitem{osband2016deep}
I.~Osband, C.~Blundell, A.~Pritzel, and B.~Van~Roy, ``Deep exploration via
  bootstrapped dqn,'' in {\em Advances in neural information processing
  systems}, pp.~4026--4034, 2016.

\bibitem{osband2018randomized}
I.~Osband, J.~Aslanides, and A.~Cassirer, ``Randomized prior functions for deep
  reinforcement learning,'' in {\em Advances in Neural Information Processing
  Systems}, pp.~8617--8629, 2018.

\bibitem{agarwal2019striving}
R.~Agarwal, D.~Schuurmans, and M.~Norouzi, ``Striving for simplicity in
  off-policy deep reinforcement learning,'' {\em arXiv preprint
  arXiv:1907.04543}, 2019.

\bibitem{kumar2019stabilizing}
A.~Kumar, J.~Fu, M.~Soh, G.~Tucker, and S.~Levine, ``Stabilizing off-policy
  q-learning via bootstrapping error reduction,'' in {\em Advances in Neural
  Information Processing Systems}, pp.~11761--11771, 2019.

\bibitem{kuss2004gaussian}
M.~Kuss and C.~E. Rasmussen, ``Gaussian processes in reinforcement learning,''
  in {\em Advances in neural information processing systems}, pp.~751--758,
  2004.

\bibitem{engel2005reinforcement}
Y.~Engel, S.~Mannor, and R.~Meir, ``Reinforcement learning with gaussian
  processes,'' in {\em Proceedings of the 22nd international conference on
  Machine learning}, pp.~201--208, 2005.

\bibitem{kuss2006gaussian}
M.~Kuss, {\em Gaussian process models for robust regression, classification,
  and reinforcement learning}.
\newblock PhD thesis, echnische Universit{\"a}t Darmstadt Darmstadt, Germany,
  2006.

\bibitem{levine2011nonlinear}
S.~Levine, Z.~Popovic, and V.~Koltun, ``Nonlinear inverse reinforcement
  learning with gaussian processes,'' in {\em Advances in Neural Information
  Processing Systems}, pp.~19--27, 2011.

\bibitem{grande2014sample}
R.~Grande, T.~Walsh, and J.~How, ``Sample efficient reinforcement learning with
  gaussian processes,'' in {\em International Conference on Machine Learning},
  pp.~1332--1340, 2014.

\bibitem{fan2018efficient}
Y.~Fan, L.~Chen, and Y.~Wang, ``Efficient model-free reinforcement learning
  using gaussian process,'' {\em arXiv preprint arXiv:1812.04359}, 2018.

\end{thebibliography}
\bibliographystyle{ieeetr}%

\newpage
\appendix
\section{Visualization of Q-Landscape}
\label{vis_Q}
Figure~\ref{landscape_pendulum} shows the visualization of learned policies (actions given different states) and $Q$ values in TD3 during training in the Pendulum-v0 environment, where the state space is 3-dim and action space is 1-dim. The red lines indicates the selected action by the current policy. The learned $Q$ function are always \textbf{non-convex} and \textbf{locally convex}. As a consequence, in many states the TD3 is not able to find globally optimal solution and local gradient information may be misleading in finding actions with the highest $Q$ values.

\begin{figure}[h]
	\centering
	\includegraphics[width=1.0\linewidth]{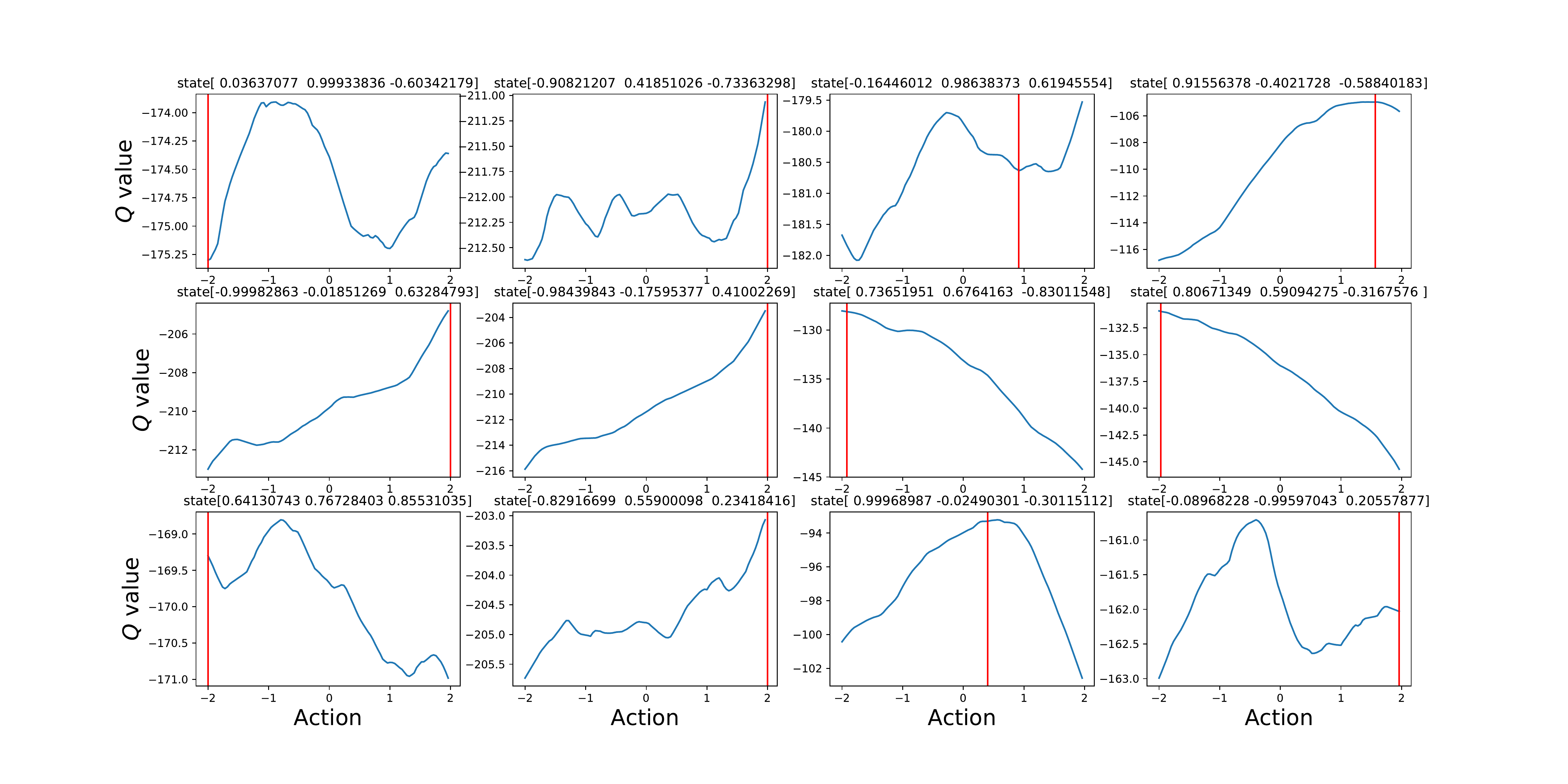}
\caption{Landscape of learned value function in the Pendulum-v0 environment}
\label{landscape_pendulum}
\end{figure}

% \begin{figure}[h]
% \begin{minipage}[htbp]{0.47\linewidth}
% 	\centering
% 	\includegraphics[width=1\linewidth]{figs/landscape_swimmer.png}
% \end{minipage}%
% \begin{minipage}[htbp]{0.47\linewidth}
% 	\centering
% 	\includegraphics[width=1\linewidth]{figs/landscape_swimmer2.png}
% \end{minipage}
% \caption{Landscape of the Swimmer-v2 environment}
% \label{landscape_swimmer}
% \end{figure}

\section{One-step Zeroth-Order Optimization with Consistent Iteration}
\begin{algorithm}[H]
	\caption{One-step Zeroth-Order Optimization with Consistent Iteration}\label{algo_zoa_2}
	\begin{algorithmic}
		\STATE \textbf{Require}
        \STATE Objective function $Q$, domain $\mathcal{A}$, current point $a_0$, number of local samples $n_1$, number of global samples $n_2$, local scale $\eta > 0$ and step size $h$, number of steps $m$.
    \FOR{$t = 1, \dots n_2$}
        \STATE \textbf{Globally sampling}
		\STATE Sample a point uniformly in the entire space by
		$$
		    a_{t0} \sim \mathcal{U}_{\mathcal{A}}
		$$
		\STATE where $\mathcal{U}_{\mathcal{A}}$ is the uniform distribution over $\mathcal{A}$.
		\FOR{$i = 1, \dots, m$}
    		\STATE \textbf{Locally sampling}
    		\STATE Sample $n_1$ points around $a_{t,i-1}$ by
    		$$
    		    \tilde{a}_{j} = a_{t, i - 1} + \mu e_j \text{ for } e_j \sim \mathcal{N}(0, I_d), j = 1, \dots n_1,
    		$$
    		\STATE where $\mathcal{N}(0, I_d)$ is the standard normal distribution centered at 0.
    		\STATE \textbf{Update}
    		\STATE Set $a_{t, i} = a_{t, i - 1} + h(\argmax_{a \in \{\tilde{a}_j\}} Q(a) - a_{t, i - 1})$%\textcolor{red}{F->Q, max->arg}.
		\ENDFOR
	\ENDFOR
	\RETURN{\ $\max_{a \in \{a_tm\}_{t = 1}^{n_2}} Q(a)$}.
	\end{algorithmic}
\end{algorithm}
\label{app:thm}

\section{Missing Proofs in Section \ref{sec:thm}}
\label{app:proof}
\paragraph{Proof of Lemma \ref{lem:GE}.} By the properties of Euclidean metric, we have the following lemma.
\begin{lemma}
Let $\mathcal{A} \subset \mathbb{R}^d$ with Euclidean metric. Any $a \in \mathcal{A}$, $\|x\|_2 \leq k$. Then there exists a set $\{a_1, \dots, a_N\}$ that is a $\frac{2k\sqrt{d}}{N^{1/d}}$-covering of $N$. In other words, for all $a \in \mathcal{A}$, $\exists i \in [N]$, 
$$
    \|a - a_i\|_2 \leq \frac{2k\sqrt{d}}{N^{1/d}}.
$$
\end{lemma}

Define $\mathcal{A}_i$ as the set of all the actions that is closest to $a_i$:
$$
    \mathcal{A}_i \overset{\Delta}{=} \left\{a \in \mathcal{A}: i = \min\{{\argmin_{j \in [N]} \|a_j - a\|_2}\}\right\}.
$$

 Let $P(\cdot)$ be the probability measure of uniform distribution over $\mathcal{A}$ We have $P(\mathcal{A}_i) \geq 1/(2N)$. Now since we have $n$ uniform samples from set $\mathcal{A}$. Let $N = n/\log^2(n/\delta)$. 
 
 \begin{lemma} [Coupon Collector's problem]
 \label{lem:CCP}
It takes $O(N \log^2(N/\delta))$ rounds of random sampling to see all $N$ distinct options with a probability at least $1-\delta$.
\end{lemma}

{\it Proof.} Consider a general sampling problem: for any finite set $\mathcal{N}$ with $|\mathcal{N}| = N$. For any $n$, whose sampling probability is $p(c)$, with a probability at least $1-\delta$, it requires at most 
$$
    \frac{\log(1/\delta)}{\log(1 + \frac{p(n)}{1-p(n)})} \text{ for $n$ to be sampled}.
$$

Since  $\log(1+x) \geq x - \frac{1}{2}x^2$ for all $x > 0$, we have 
$$
    \frac{\log(1/\delta)}{\log(1 + \frac{p(n)}{1-p(n)})} \leq \log(1/\delta) \frac{1}{\frac{p(n)}{1-p(n)} - \frac{p(n)^2}{2(1-p(n))^2}} = O(\log(1/\delta) \frac{1-p(n)}{p(n)}).
$$
Searching the whole space $\mathcal{N}$ with each new element being found with probability $\frac{N - i}{N}$ at round $i$, it requires at most 
$$
    O(\sum_{i = 1}^N \log(\frac{N}{\delta}) \frac{N}{N - i}) = O( \log^2(\frac{N}{\delta})N),
$$
with a probability at most $1-\delta$.

By Lemma \ref{lem:CCP} We have with a probability at least $1-\delta$, there exists a sample in each $\mathcal{A}_i$ described above. To proceed, we apply the Lipschitz of the $Q$ function, Lemma \ref{lem:GE}
 follows. 
 
\paragraph{Proof of Theorem \ref{thm:1}.} Now we proceed to show Theorem \ref{thm:1}. We denote $\frac{2k\sqrt{d}L\log^{1/d}(n/\delta)}{n^{1/d}}$ by $\sigma^2$. Our global policy network gives a prediction from a linear model. Let our dataset be $\{s_t, a_t^+\}^T_{t = 1}$. Let $\epsilon_t = a_t^+ - a_t^*$. Let $\vect{S} = (s_1, \dots, s_T)^T$ and $\vect{a}^+, \vect{a}^*, \vect{\epsilon}$ be the corresponding vector for $(a_t^+)_{t = 1}^T, (a_t^*)_{t = 1}^T, (\epsilon_t)_{t = 1}^T$. We have any $\epsilon_i, \epsilon_j$ are independent for $i \neq j$. We further make an assumption that $\E[\epsilon_t] = 0$. Then we immediately have $\E [\epsilon_i \epsilon_j] = 0$ as well. This assumption is just for simplifying the proof, we can show similar results without assume the unbiasedness as the bias can be bounded. 

To proceed, let 
$
    \theta^* = \argmin_{\theta} \E_s\| s^T\theta -  a^*(s)\|_2^2.
$

Since the ERM solution is simply OLS (ordinary least square). We have the estimate
$$
    \hat \theta = (\vect{S}^T\vect{S})^{-1}\vect{S}^T \vect{a}^+.
$$
We have 
$$
    \hat \theta - \theta^* = (\vect{S}^T\vect{S})^{-1}\vect{S}^T(\vect{a}^* - \vect{S}\theta^*) + (\vect{S}^T\vect{S})^{-1}\vect{S}^T\vect{\epsilon}.
$$
Since $\|\epsilon_t\|^2 \leq \sigma^2$, we have $\operatorname{Var}(\epsilon_t) \leq \sigma^2$. We observe that 
$$
    \operatorname{Var}((\vect{S}^T\vect{S})^{-1}\vect{S}^T\vect{\epsilon}) \leq \sigma^2 ((\vect{S}^T\vect{S})^{-1}) = \mathcal{O}(\sigma^2 p / T).
$$
The generalization error is given by
\begin{align*}
    &\quad \E_s\|s^T\hat\theta - a^*(s)\|_2^2\\
    &\leq \E_s\|s^T\hat\theta - s^T \theta^*\|_2^2 + \E_s\| s^T\theta^* -  a^*(s)\|_2^2 \\
    &= \mathcal{O}\left( \frac{\sigma^2p}{T} +  \E_s\|s^T(\vect{S}^T\vect{S})^{-1}\vect{S}^T(\vect{a}^* - \vect{S}\theta^*)\|_2^2 + \E_s\| s^T\theta^* -  a^*(s)\|_2^2 \right) \\
    &= \tilde{\mathcal{O}}\left(\frac{\sigma^2p}{T} + \E_s\| s^T\theta^* -  a^*(s)\|_2^2\right).
\end{align*}

\section{Implementation Details}
\label{appd_zospi_mdn}
\subsection{Network Structure And Hyper-Params}
In our experiments, we follow ~\cite{fujimoto2018addressing} to use a $3$-layer MLP with $256$ hidden units for both critic and actor networks. We also follow ~\cite{fujimoto2018addressing} to use $25000$ timesteps for worm-up and use a batch-size of $256:1$ training-interaction proportion during training.

\subsection{Mixture Density Networks}

Our implementation of ZOSPI with MDN is based on neural network with multiple outputs. For MDN with $K$ Gaussian mixture outputs, the neural network has $3\times K$-dim output. The first $K$-dim units are normalized with softmax activation as the probability of selecting the Gaussians, the following $K$-dim units are corresponding $K$ mean values of the Gaussians, and the last $K$-dim units are standard deviation of the $K$ Gaussians. In our experiments, we use Diracs instead of Gaussians for parameterization. Therefore, our networks output $2\times K$-dim units for each action dimension, where the first $K$-dimensions denote the probabilities and the latter $K$-dimensions denote the mean values. In our experiments, we use $K=1$ for Hopper, HalfCheetah and Ant, $K=5$ for Walker2d, and $K=5$ for Humanoid. ($K=5$ and $K=10$ achieve on-par performance for the Humanoid environment, though the $K=10$ setting spend roughly one more time computational expense).

More implementation details are provided with the code in the supplementary material.

\subsection{Running Time of ZOSPI}
We conduct our experiments with 8 GTX TITAN X GPUs and 32 Intel(R) Xeon(R) E5-2640 v3 @ 2.60GHz CPUs. The wall clock time of our proposed method is roughly 3-times slower than running TD3, without application of MDNs (i.e., $\text{NoD} = 1$). It takes roughly $20$ hours to train Hopper with $10$ seeds, and takes about $120$ hours to train Humanoid when $\text{NoD}=10$ with $10$ seeds.

\section{Better Exploration with Bootstrapped Networks}

Sample efficient RL requires algorithms to balance exploration and exploitation. One of the most popular way to achieve this is called optimism in face of uncertainty (OFU)~\cite{brafman2002r,jaksch2010near,azar2017minimax,jin2018q}, which gives an upper bound on $Q$ estimates and applies the optimal action corresponding to the upper bound. The optimal action $a_t$ is given by the following optimization problem:
\begin{equation}
    \argmax_a  Q^+(s_t, a), \label{equ:argmaxQ}
\end{equation}
where $Q^+$ is the upper confidence bound on the optimal $Q$ function. %Contrasted to DPG with a random noise on action space, OFU can has a guaranteed sample efficiency in tabular setting \cite{jaksch2010near}. 
A guaranteed exploration performance requires both a good solution for (\ref{equ:argmaxQ}) and a valid upper confidence bound. 

While it is trivial to solve (\ref{equ:argmaxQ}) in the tabular setting, the problem can be intractable in a continuous action space. Therefore, as shown in the previous section, ZOSPI adopts a local set to approximate policy gradient descent methods in the local region and further applies a global sampling scheme to increase the potential chance of finding a better maxima.

As for the requirement of a valid upper confidence bound, we use bootstrapped $Q$ networks to address the uncertainty of $Q$ estimates as in~\cite{osband2016deep,osband2018randomized,agarwal2019striving,kumar2019stabilizing,ciosek2019better}. Specifically, we keep $K$ estimates of $Q$, namely $Q_1, \dots Q_K$ with bootstrapped samples from the replay buffer. Let $\overline{Q} = \frac{1}{K} \sum_{k} Q_k(s, a)$. An upper bound $Q^+$ is %given by
\begin{equation}
     Q^+(s, a) = \overline{Q} + \phi \sqrt{\frac{1}{K} \sum_k [Q_k(s, a) - \overline{Q}]^2}, \label{equ:Q+}
\end{equation}
where $\phi$ is the hyper-parameter controlling the failure rate of the upper bound. Another issue is on the update of bootstrapped $Q$ networks. Previous methods~\cite{agarwal2019striving} usually update each $Q$ network with the following target
$r_{t}+\gamma Q_k\left(s_{t+1},\pi_{\theta_t}(s_{t+1})\right),$ which violates the Bellman equation as $\pi_{\theta_t}$ is designed to be the optimal policy for $Q^+$ rather than $Q_k$. Using $\pi_{\theta_t}$ also introduces extra dependencies among the $K$ estimates. We instead %found that a random global sampling achieves better performance: 
employ a global random sampling method to correct the violation as
$$
    r_{t}+\gamma \max_{i = 1, \dots n} Q_k\left(s_{t+1}, a_i\right), \quad a_1, \dots a_n \sim \mathcal{U}_{\mathcal{A}}.
$$
The correction also reinforces the argument that a global random sampling method yields a good approximation to the solution of the optimization problem (\ref{equ:argmaxQ}). The detailed algorithm is provided in Algorithm \ref{algo_SPI_UCB} in Appendix~\ref{append:algo}.

\subsection{Algorithm \ref{algo_SPI_UCB}: ZOSPI with Bootstrapped $Q$ networks}
\label{append:algo}
\begin{algorithm}[ht]
\small
	\caption{ZOSPI with UCB Exploration}\label{algo_SPI_UCB}
	\begin{algorithmic}
		\STATE \textbf{Require}
		\begin{itemize}
		    \item The number of epochs $M$, the size of mini-batch $N$, momentum $\tau > 0$ and the number of Bootstrapped Q-networks $K$.
			\item Random initialized policy network $\pi_{\theta_1}$, target policy network $\pi_{\theta'_1}$, $\theta'_1\leftarrow  \theta_1$.
		    \item $K$ random initialized $Q$ networks, and corresponding target networks, parameterized by $w_{k, 1},w'_{k, 1}$, $w_{k, 1}'\leftarrow w_{k, 1}$ for $k = 1, \dots, K$.
		\end{itemize}
		\FOR{iteration $= 1,2,...$}
		\FOR{t $= 1,2,..., T$}
		\STATE $\#$ Interaction
		\STATE Run policy $\pi_{\theta'_t}$, and collect transition tuples $(s_t,a_t,s'_t,r_t, m_t)$.
		\FOR{epoch $j = 1,2,..., M$}
		\STATE Sample a mini-batch of transition tuples $\mathcal{D}_j = \{(s,a,s',r,m)_{i}\}_{i = 1}^N$.
		\STATE $\#$ Update $Q$
	    \FOR{$k=1,2,...,K$}
	    \STATE Calculate the $k$-th target $Q$ value $y_{ki} = r_i + \max_l Q_{w'_{k,t}}(s'_i,a'_l)$, where $a'_l \sim \mathcal{U}_\mathcal{A}$.
	    \STATE Update $w_{k, t}$ with loss $\sum_{i = 1}^N m_{ik}(y_{ki} - Q_{w'_{k, t}}(s_i,a_i))^2$.
	    \ENDFOR
		\STATE $\#$ Update $\pi$
	    \STATE Calculate the predicted action $a_0 = \pi_{\theta'_t}(s_i)$
	    \STATE Sample actions $a_l\sim \mathcal{U}_\mathcal{A}$
	    \STATE Select $a^+ \in \{a_l\} \cup \{a_0\}$ as the action with maximal $Q^+(s_t, a)$ defined in (\ref{equ:Q+}).
	    \STATE Update policy network with Eq.(\ref{equ:0th}).
		\ENDFOR
		\STATE $\theta'_{t+1}\leftarrow \tau \theta_{t} + (1-\tau) \theta'_{t}$.
		\STATE $w_{k, t+1}'\leftarrow \tau w_{k, t} + (1-\tau) w_{k,t}'$.
		\STATE $w_{k, t+1} \leftarrow w_{k, t}$; $\theta_{t+1} \leftarrow \theta_{t}$.
		\ENDFOR
		\ENDFOR
		
	\end{algorithmic}
\end{algorithm}%%%%%%%%%%%%%%%%%%%%%%%%%%%%%%%%%%%%%%%%%%%%%%%%%%%%%%%%

\section{Gaussian Processes for Continuous Control}
Different from previous policy gradient methods, the self-supervised learning paradigm of ZOSPI permits it to learn both its actor and critic with a regression formulation. Such a property enables the learning of actor in ZOSPI to be implemented with either parametric models like neural networks or non-parametric models like Gaussian Processes (GP). Although plenty of previous works have discussed the application of GP in RL by virtue of its natural uncertainty capture ability, most of these works are limited to model-based methods or discrete action spaces for value estimation~\cite{kuss2004gaussian,engel2005reinforcement,kuss2006gaussian,levine2011nonlinear,grande2014sample,fan2018efficient}. On the other hand, ZOSPI formulates the policy optimization in continuous control tasks as a regression objective, therefore empowers the usage of GP policy in continuous control tasks. %To the best of our knowledge, GP with ZOSPI extends those branch of works in the domain of continuous control for the first time.

As a first attempt of applying GP policies in continuous control tasks, we simply alter the actor network with a GP to interact with the environment and collect data, while the value approximator is still parameterized by a neural network. We leave the investigation of better consolidation design in future work.

\section{Experiments on the Four-Solution-Maze.}

\begin{figure*}[t]
\centering
\subfigure[Env. and the optimal policy]{
\begin{minipage}[b]{0.149\linewidth}
\label{four_way_maze}
\includegraphics[width=0.98\linewidth]{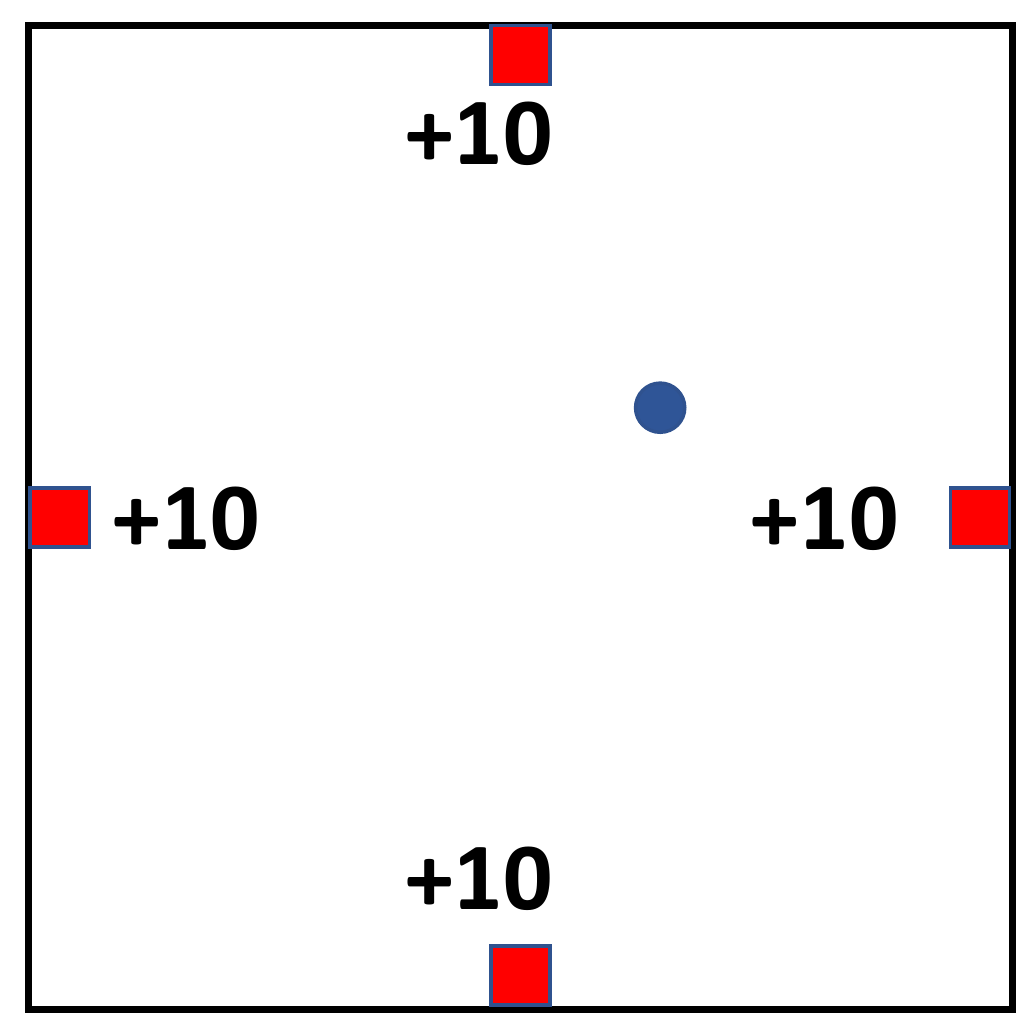}
\end{minipage}
\begin{minipage}[b]{0.151\linewidth}
\includegraphics[width=0.98\linewidth]{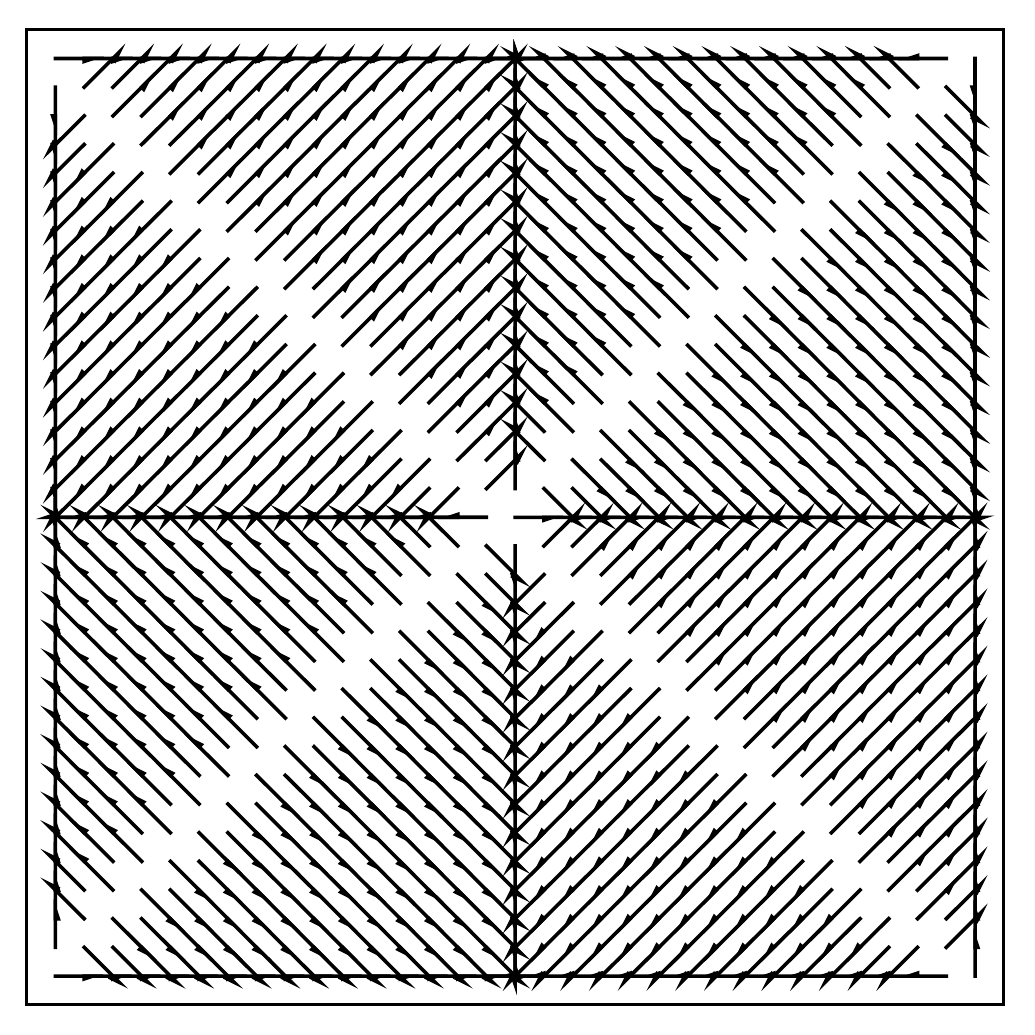}
\end{minipage}}
\subfigure[Performance comparison]{
\begin{minipage}[b]{0.32\linewidth}
\label{toy_curve}
\includegraphics[width=0.98\linewidth]{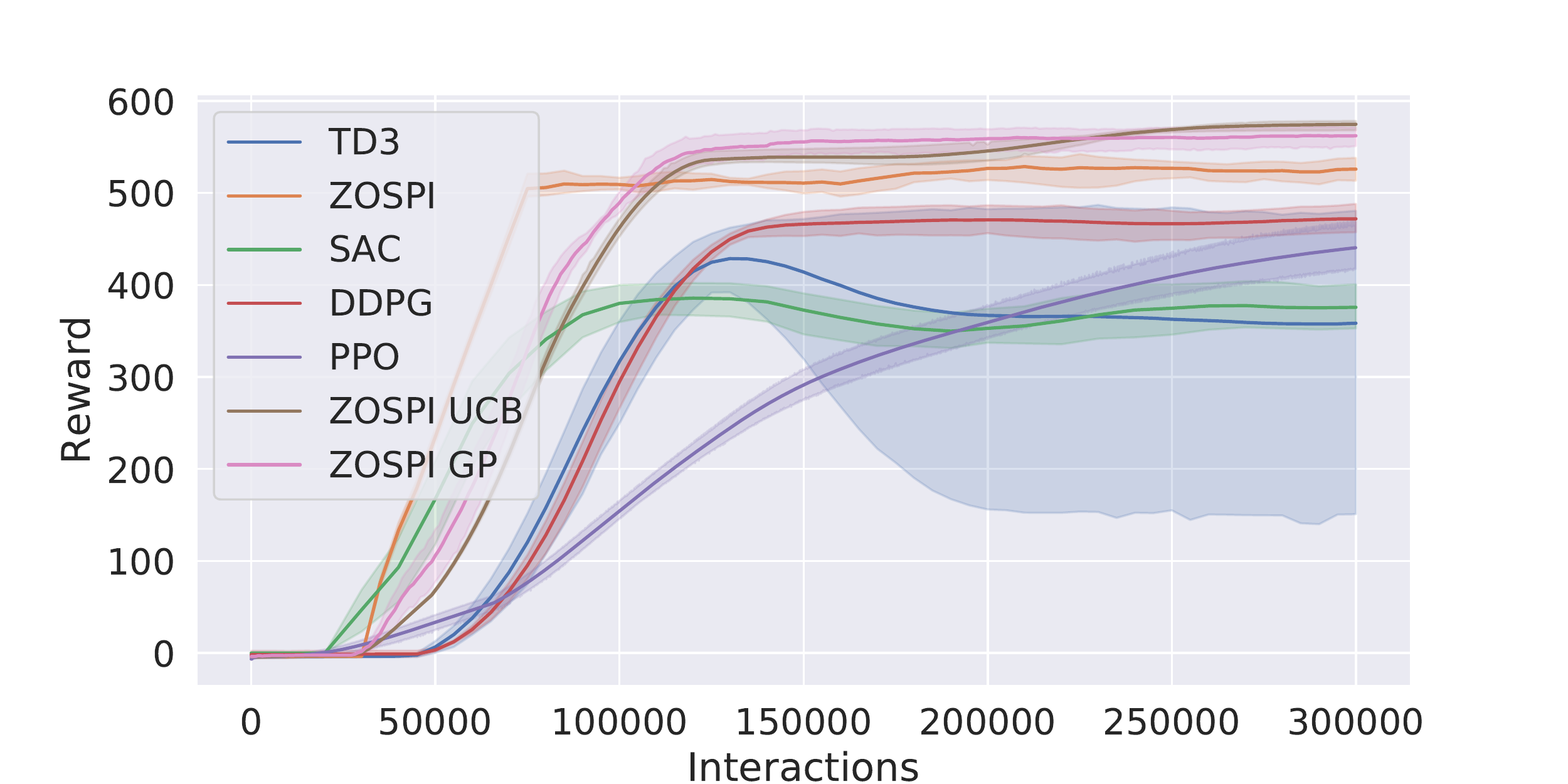}%\vspace{4pt}
\end{minipage}}
\subfigure[PPO]{
\begin{minipage}[b]{0.32\linewidth}
\label{vis_toy_start}
\includegraphics[width=0.98\linewidth]{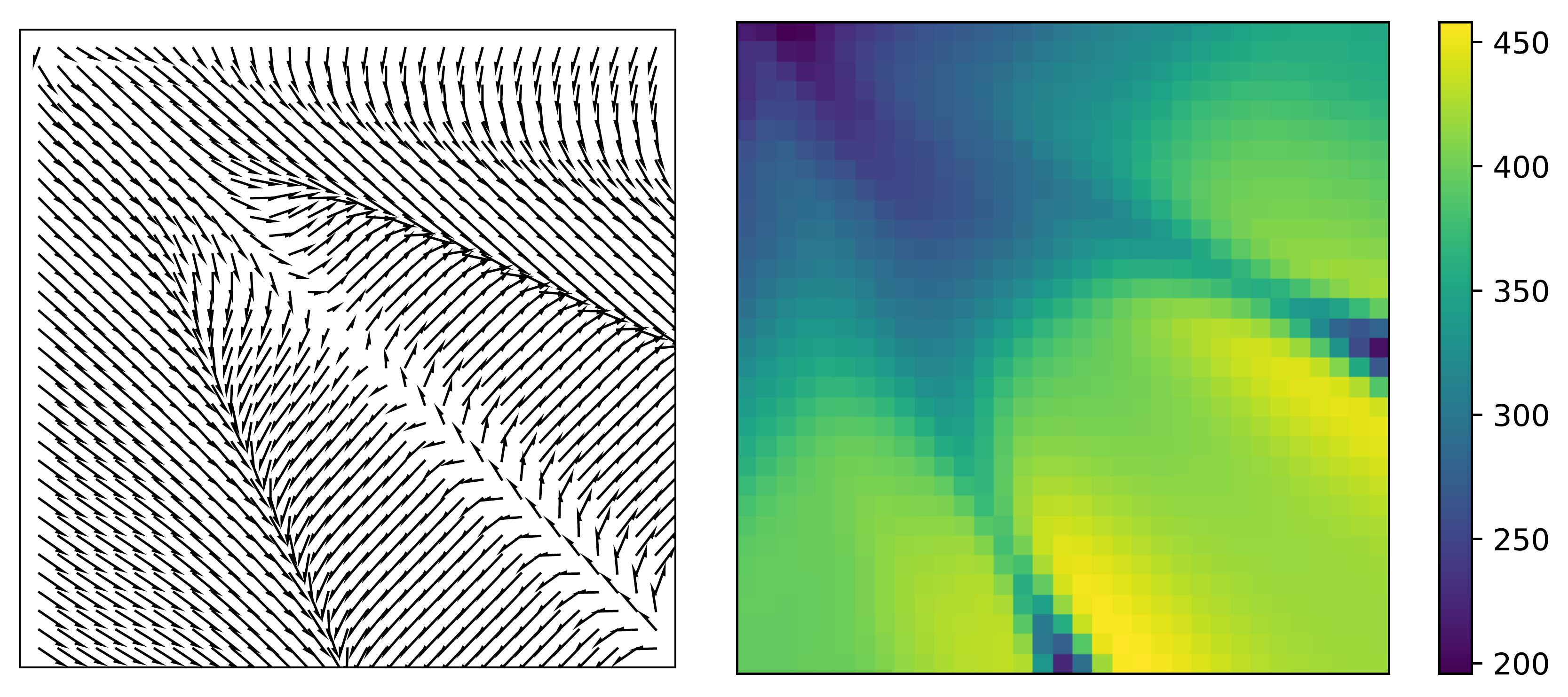}%\vspace{4pt}
\end{minipage}}\\
\subfigure[DDPG]{
\begin{minipage}[b]{0.32\linewidth}
\includegraphics[width=0.98\linewidth]{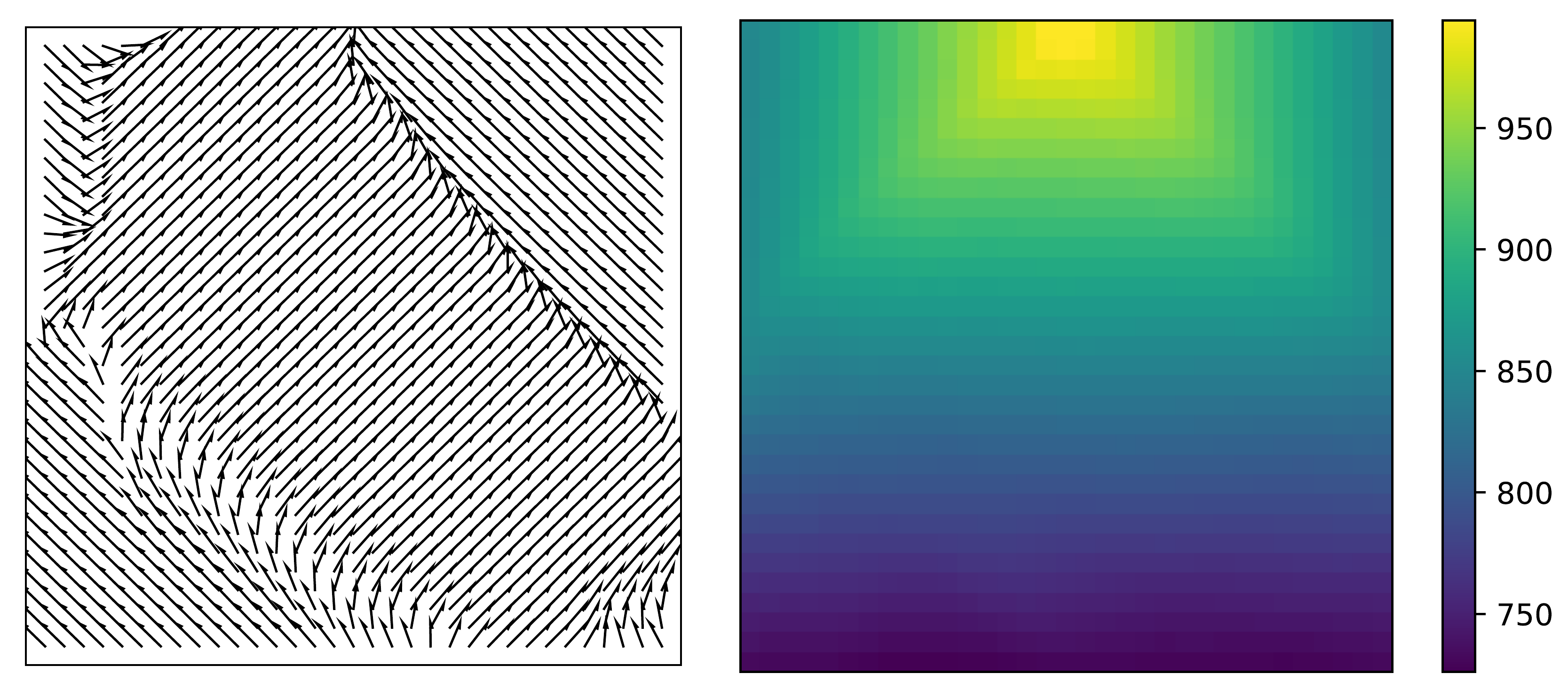}%\vspace{4pt}
\end{minipage}}
\subfigure[TD3]{
\begin{minipage}[b]{0.32\linewidth}
\includegraphics[width=0.98\linewidth]{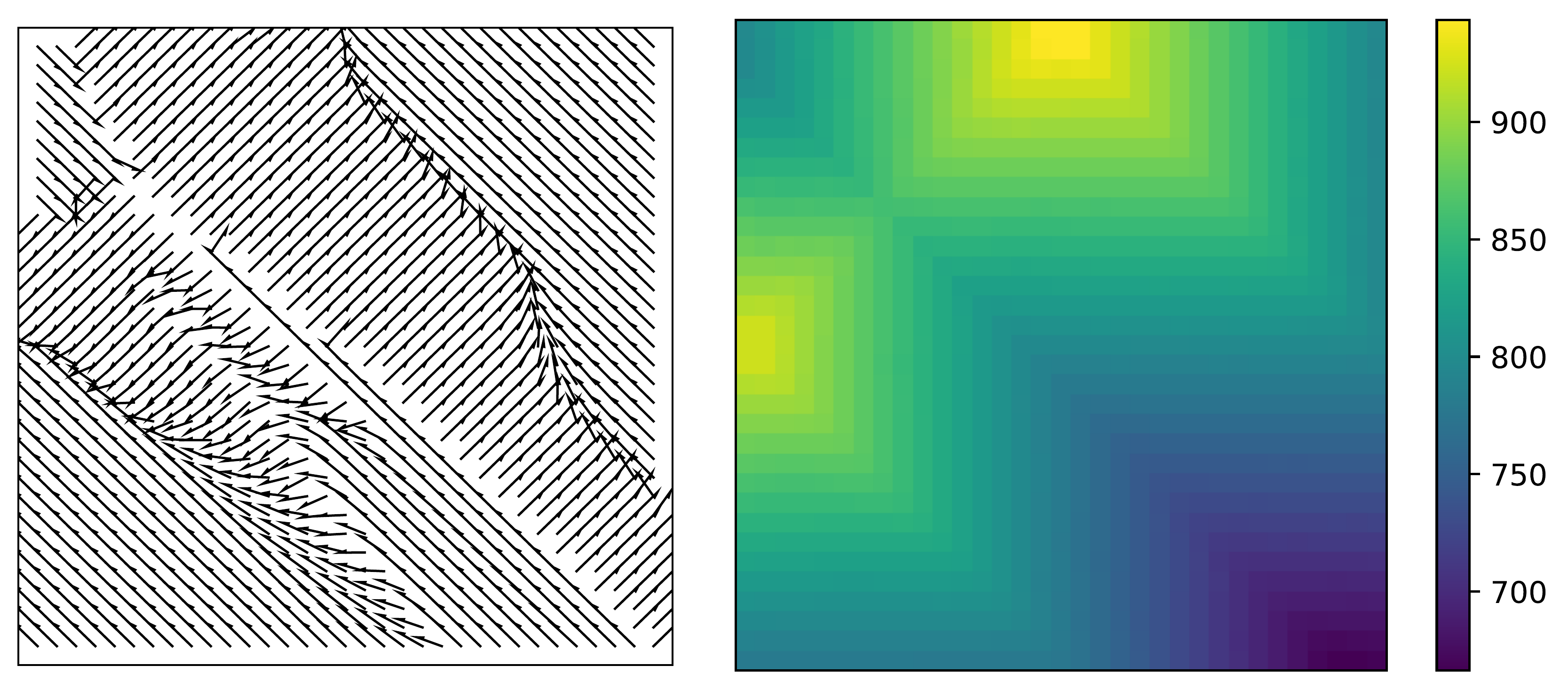}%\vspace{4pt}
\end{minipage}}
\subfigure[SAC]{
\begin{minipage}[b]{0.32\linewidth}
\includegraphics[width=0.98\linewidth]{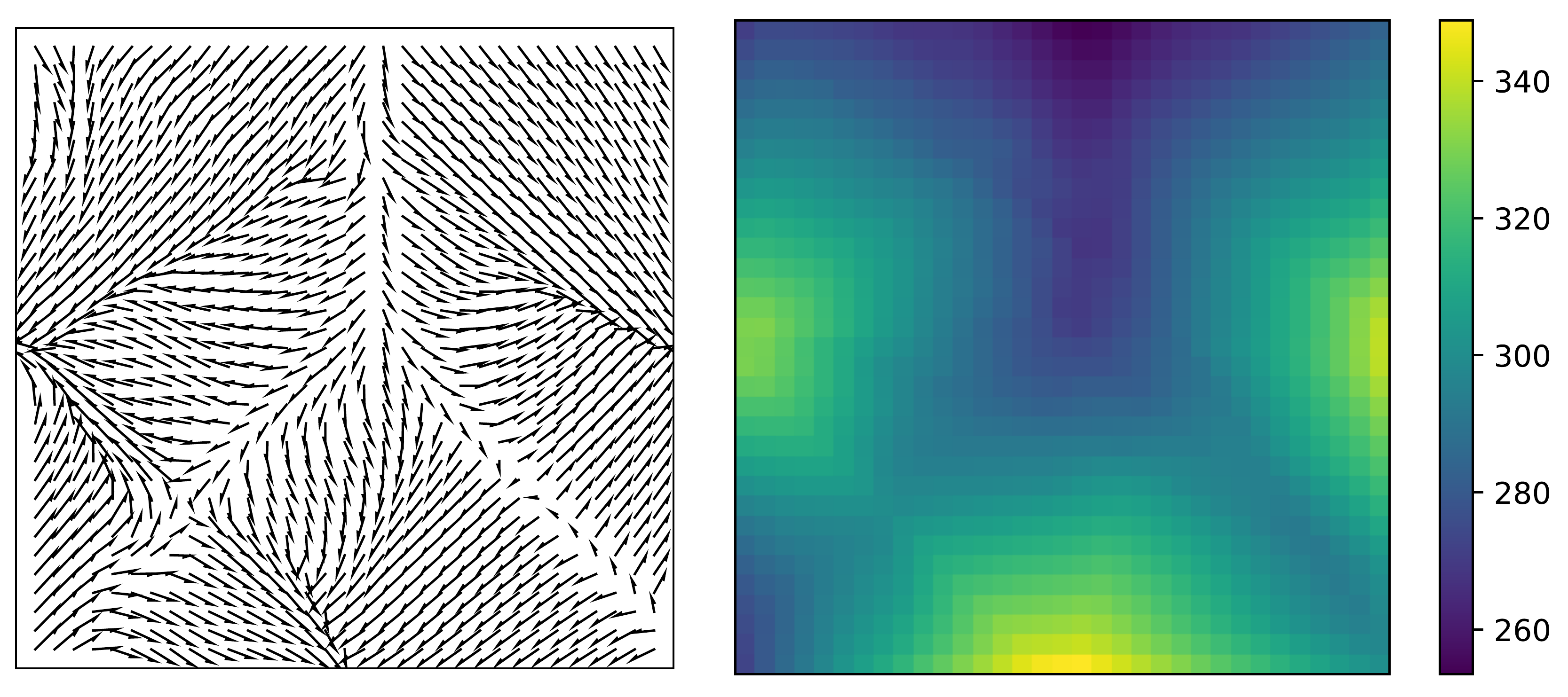}%\vspace{4pt}
\end{minipage}}\\
\subfigure[ZOSPI]{
\begin{minipage}[b]{0.32\linewidth}
\includegraphics[width=0.98\linewidth]{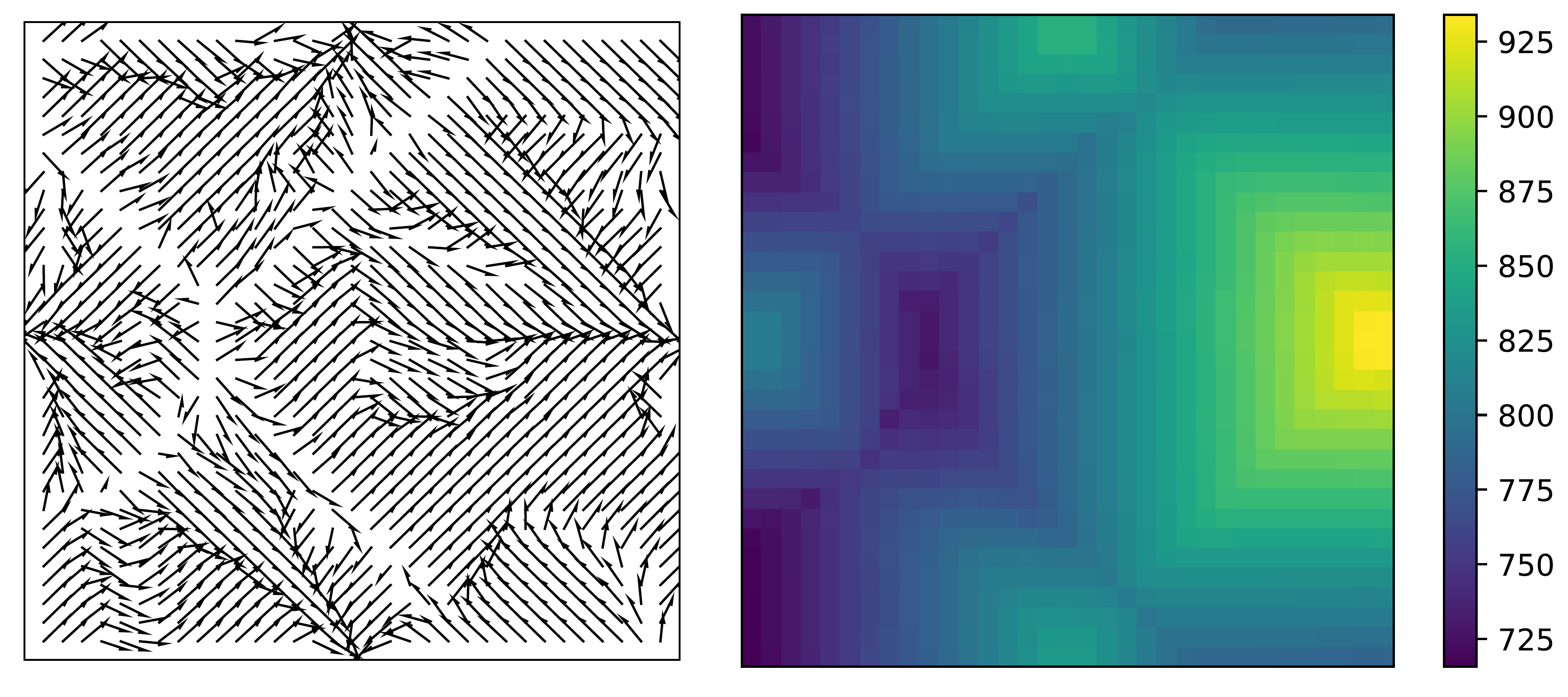}%\vspace{4pt}
\end{minipage}}
\subfigure[ZOSPI-UCB]{
\begin{minipage}[b]{0.32\linewidth}
\includegraphics[width=0.98\linewidth]{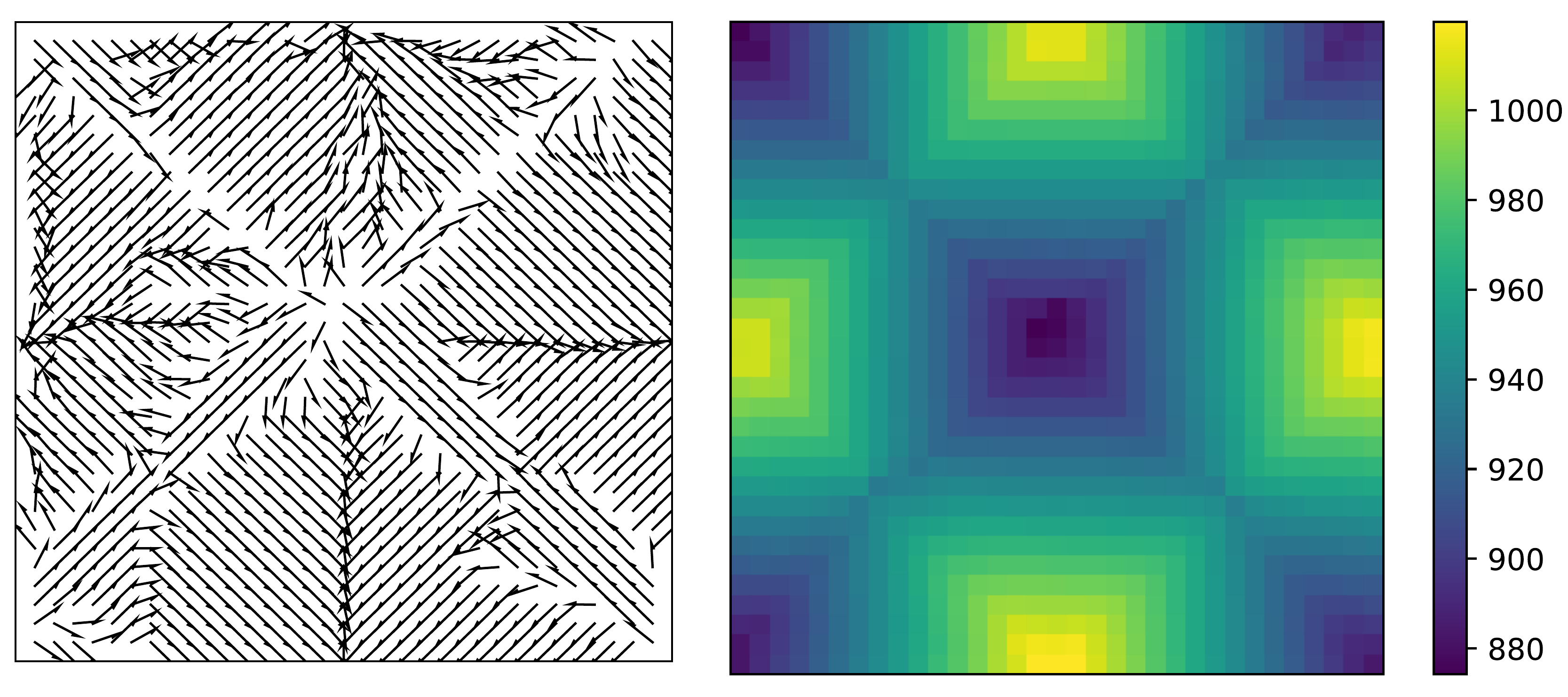}%\vspace{4pt}
\end{minipage}}
\subfigure[ZOSPI-GP]{
\begin{minipage}[b]{0.32\linewidth}
\label{vis_toy_end}
\includegraphics[width=0.98\linewidth]{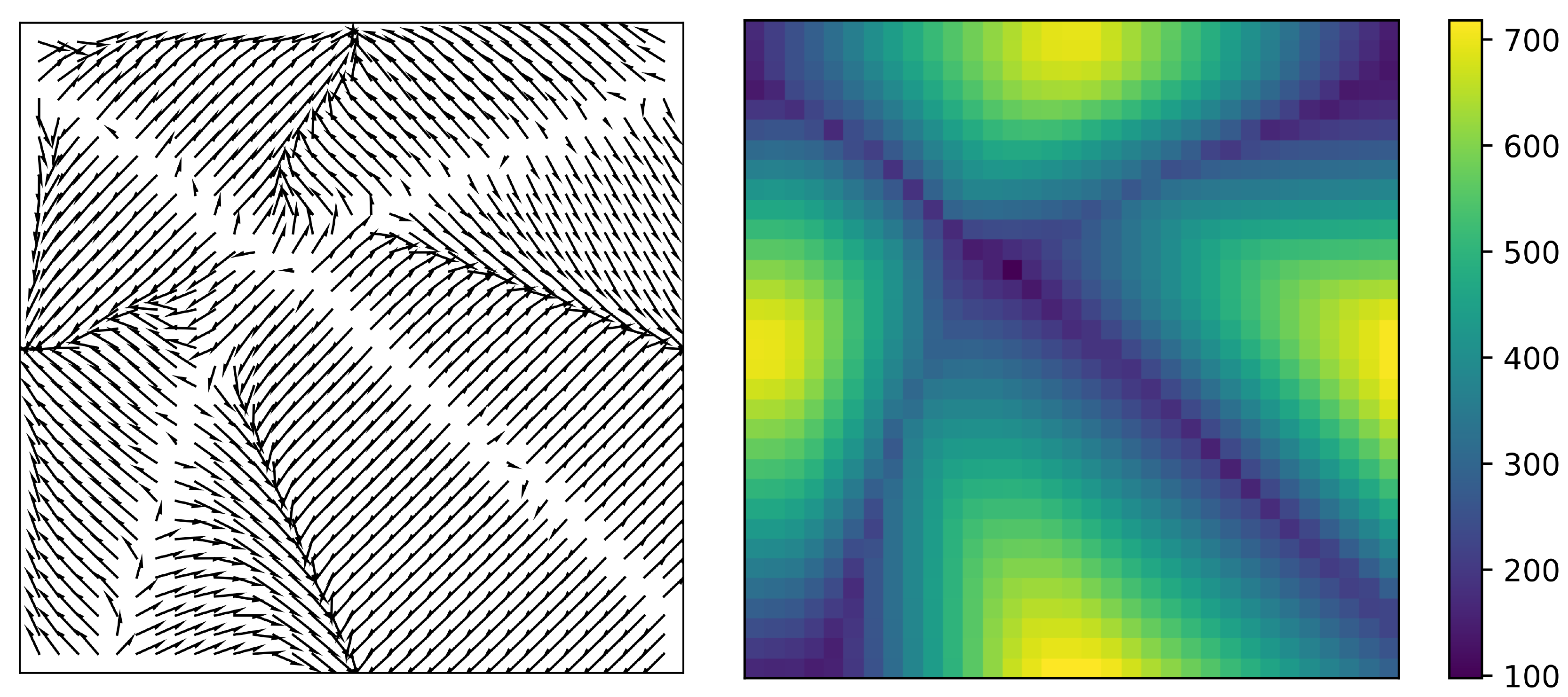}%\vspace{4pt}
\end{minipage}}\\
\caption{Visualization of learned policies on the FSM environment. (a) the FSM environment and its optimal solution, where the policy should find the nearest reward region and move toward it; (b) learning curves of different approaches; (c)-(i) visualize the learned policies and corresponding value functions. We run multiple repeat experiments and show the most representative and well-performing learned value function and policy of each method.}
%\vspace{-0.9cm}
\end{figure*}
% \begin{figure}[t]
% \begin{minipage}[htbp]{0.33\linewidth}
% 	\centering
% 	\includegraphics[width=1\linewidth]{figs/toy_ppo.pdf}
% \end{minipage}%
% \begin{minipage}[htbp]{0.33\linewidth}
% 	\centering
% 	\includegraphics[width=1\linewidth]{figs/toy_ddpg.pdf}
% \end{minipage}%
% \begin{minipage}[htbp]{0.33\linewidth}
% 	\centering
% 	\includegraphics[width=1\linewidth]{figs/toy_td3.pdf}
% \end{minipage}\\%
% \begin{minipage}[htbp]{0.33\linewidth}
% 	\centering
% 	\includegraphics[width=1\linewidth]{figs/toy_sac.pdf}
% \end{minipage}%
% \begin{minipage}[htbp]{0.33\linewidth}
% 	\centering
% 	\includegraphics[width=1\linewidth]{figs/toy_spi.pdf}
% \end{minipage}%
% \begin{minipage}[htbp]{0.33\linewidth}
% 	\centering
% 	\includegraphics[width=1\linewidth]{figs/toy_UCB.pdf}
% \end{minipage}\\%
% \caption{}
% \label{vis_toy}
% \end{figure}
The Four-Solution-Maze (FSM) environment is a diagnostic environment where four positive reward regions with a unit side length are placed in the middle points of $4$ edges of a $N\times N$ map. An agent starts from a uniformly initialized position in the map and can then move in the map by taking actions according to the location observations (current coordinates $x$ and $y$). Valid actions are limited to $[-1,1]$ for both $x$ and $y$ axes. Each game consists of $2N$ timesteps for the agent to navigate in the map and collect rewards. In each timestep, the agent will receive a $+10$ reward if it is inside one of the $4$ reward regions or a tiny penalty otherwise. For simplicity, there are no obstacles in the map, the optimal policy thus will find the nearest reward region, directly move towards it, and stay in the region till the end. Figure~\ref{four_way_maze} visualizes the environment and the ground-truth optimal solution. 

Although the environment is simple, we found it extremely challenging due to existence of multiple sub-optimal policies that only find some but not all four reward regions. We do not conduct grid search on hyper-parameters of the algorithms compared in our experiments but set them to default setting across all experiments. Though elaborated hyper-parameter tuning may benefit for certain environment.

On this environment we compare ZOSPI to on-policy and off-policy SOTA policy gradient methods in terms of the learning curves, each of which is averaged by $5$ runs. The results are presented in Figure~\ref{toy_curve}. And learned policies from different methods are visualized in Figure~\ref{vis_toy_start}-\ref{vis_toy_end}. For each method we plot the predicted behaviors of its learned policy at grid points using arrows (although the environment is continuous in the state space), and show the corresponding value function of its learned policy with a colored map. All policies and value functions are learned with $0.3 \mathrm{M}$ interactions except for SAC whose figures are learned with $1.2\mathrm{M}$ interactions as it can find $3$ out of $4$ target regions when more interactions are provided.

We use $4$ bootstrapped $Q$ networks for the upper bound estimation in consideration of both better value estimation and computational cost for ZOSPI with UCB. And in ZOSPI with GP, a GP model is used to replace the actor network in data-collection, $i.e.$,exploration.
The sample efficiency of ZOSPI is much higher than that of other methods. Noticeably ZOSPI with UCB exploration is the only method that can find the optimal solution, \emph{i.e.}, a policy directs to the nearest region with a positive reward. All other methods get trapped in sub-optimal solutions by moving to only part of reward regions they find instead of moving toward the nearest one.

% On the other hand, we find ZOSPI with UCB exploration, although can find the global optimal solution, outperforms ZOSPI only by a small margin, which is at the price of keeping $K$ bootstrapped $Q$ networks and updating them through separately sampled actions. Since it is relatively inefficient in terms of computational complexity, we choose to demonstrate ZOSPI without UCB exploration on the MuJoCo benchmarks in consideration of both sample efficiency and computational efficiency.

\end{document}